\pdfoutput=1
\documentclass[letterpaper]{article} 
\usepackage{aaai24}  
\usepackage{times}  
\usepackage{helvet}  
\usepackage{courier}  
\usepackage[hyphens]{url}  
\usepackage{graphicx} 
\usepackage{natbib}  
\usepackage{caption} 
\usepackage{makecell}
\usepackage{multirow}
\usepackage{float}
\usepackage{enumitem}
\usepackage{soul}
\usepackage{colortbl}
\usepackage{tikz}
\usepackage{amsmath}
\usepackage{amssymb}
\usepackage{mathtools}
\usepackage{amsthm}
\usepackage{pifont}
\newcommand{\xmark}{\ding{55}}%
\newtheorem{prop}{Proposition}
\usepackage{amsmath}
\usepackage{cleveref}
\usepackage{url}            
\usepackage{booktabs}       
\usepackage{amsfonts}       
\usepackage{nicefrac}       
\usepackage{microtype}      
\usepackage{xcolor}         
\usepackage[textsize=tiny]{todonotes}
\usepackage{floatflt}
\usepackage{tabularx}
\usepackage{subcaption}

\usepackage{algorithm}
\usepackage{algorithmic}

\usepackage{newfloat}
\usepackage{listings}
\DeclareCaptionStyle{ruled}{labelfont=normalfont,labelsep=colon,strut=off} 
\floatstyle{ruled}
\newfloat{listing}{tb}{lst}{}
\floatname{listing}{Listing}
\newcommand{\ours}{Ours} 
\newcommand{\din}{\mathcal{D}_{\text{in}}}
\newcommand{\dintest}{\mathcal{D}_{\text{in}}^{\text{test}}}

\newcommand{\dout}{\mathcal{D}_{\text{out}}}
\newcommand{\douttest}{\mathcal{D}_{\text{out}}^{\text{test}}}

\newcommand{\Lin}{\mathcal{L}_{\text{in}}}
\newcommand{\Lout}{\mathcal{L}_{\text{out}}}

\newcommand{\x}{\boldsymbol{x}}

\title{EAT: Towards Long-Tailed Out-of-Distribution Detection}
\author{
    Tong Wei\thanks{Corresponding author}, Bo-Lin Wang, Min-Ling Zhang
}
\affiliations{
    School of Computer Science and Engineering, Southeast University, Nanjing 210096, China\\

    Key Laboratory of Computer Network and Information Integration (Southeast University), Ministry of Education, China\\
    \{weit, wangbl, zhangml\}@seu.edu.cn
}

\usepackage{bibentry}

\begin{document}

\maketitle

\begin{abstract}
Despite recent advancements in out-of-distribution (OOD) detection, most current studies assume a class-balanced in-distribution training dataset, which is rarely the case in real-world scenarios. This paper addresses the challenging task of long-tailed OOD detection, where the in-distribution data follows a long-tailed class distribution. The main difficulty lies in distinguishing OOD data from samples belonging to the tail classes, as the ability of a classifier to detect OOD instances is not strongly correlated with its accuracy on the in-distribution classes. To overcome this issue, we propose two simple ideas: (1) Expanding the in-distribution class space by introducing multiple abstention classes. This approach allows us to build a detector with clear decision boundaries by training on OOD data using \textit{virtual labels}. (2) Augmenting the context-limited tail classes by overlaying images onto the context-rich OOD data. This technique encourages the model to pay more attention to the discriminative features of the tail classes. We provide a clue for separating in-distribution and OOD data by analyzing gradient noise. Through extensive experiments, we demonstrate that our method outperforms the current state-of-the-art on various benchmark datasets. Moreover, our method can be used as an add-on for existing long-tail learning approaches, significantly enhancing their OOD detection performance. Code is available at: \url{https://github.com/Stomach-ache/Long-Tailed-OOD-Detection}.
\end{abstract}

\section{Introduction}
\label{sec:intro}

Deep neural networks (DNNs) can achieve high performance in various real-world applications by training on large-scale and well-annotated datasets. Most supervised learning literature makes a common assumption that the training and test data have the same distribution. However, DNNs in deployment often encounter data from an unknown distribution, and it has been shown that DNNs tend to produce wrong predictions on anonymous, or out-of-distribution (OOD) test data with high confidence \cite{hendrycks17baseline,liang2018enhancing,hein2019relu}, which can result in severe mistakes in practice.

Recently, OOD detection, which aims to reject OOD test data without classifying them as in-distribution labels, has caught great attention. Existing state-of-the-art OOD detectors achieve huge success by maximizing the predictive uncertainty \cite{hendrycks2018deep,meinke2019towards}, energy function \cite{liu2020energy}, and abstention class confidence \cite{mohseni2020self,chen2021atom} for OOD data. However, these approaches assume that the in-distribution data is class-balanced, which is usually violated in real-world tasks \cite{van2017devil,liu2019large,cui2019class,weit2020tnnls,wei2021rolt,wei2023cvpr}. In this paper, we consider that the in-distribution training data follows a long-tailed class distribution. Under this setup, directly combining existing OOD detectors with long-tailed learning methods still leads to unsatisfactory performance \cite{wang2022pascl}. 
So, a natural question is raised:
\begin{center}
{\it Is it possible to effectively distinguish OOD data from tail-class samples?}
\end{center}
To answer this question, we propose a novel framework, {EAT}, which is composed of two key ingredients: (1) \textit{dynamic virtual labels}, which expand the classification space with abstention classes for OOD data and are dynamically assigned by the model in the training process. EAT classifies OOD samples into abstention OOD classes rather than imposing uniform predictive probabilities over inlier classes such as in OE \cite{hendrycks2018deep}, Energy \cite{liu2020energy}, and PASCL \cite{wang2022pascl}. This step is critical because inherent similar OOD samples can be pushed closer if they are classified as an identical OOD class, and the decision boundary between inlier data and OOD data will be clearer. %
(2) \textit{tail class augmentation}, which augments the tail-class images by pasting them onto the context-rich OOD images to force the model to focus on the foreground objects. Precisely, given an original image from the tail class, it is cropped in various sizes and pasted onto images from OOD data. Then, we can create tail-class images with more diverse contexts by changing the background. The generalization for tail classes can be significantly improved.

To further enhance the classification of inlier data, we propose a method that involves fine-tuning the classifiers exclusively using inlier data. This fine-tuning process employs a class-balanced loss function for a few iterations. Additionally, we illustrate that our method can be seamlessly integrated with existing long-tail learning approaches, leading to a significant improvement in their OOD detection performance. This is evident from the results presented in \Cref{tab:plugin}, where our method acts as a valuable plugin to boost the performance of these approaches. These findings contradict the argument put forth by previous work \cite{vaze2022openset} that a classifier performing well on in-distribution data would automatically excel as an OOD detector.

The key \textbf{contributions} of this paper are summarized as follows:
(1) We tackle the challenging and under-explored problem of long-tailed OOD detection. This problem poses unique difficulties and requires innovative solutions. (2) We propose a novel approach to train OOD data using virtual labels, presenting an alternative to the outlier exposure method specifically designed for long-tailed data. Furthermore, we provide insights into the impact of virtual labels by examining gradient noise, deepening our understanding of their effectiveness. (3) Through extensive experiments conducted on various datasets, we empirically validate the effectiveness of our proposed method. Our results demonstrate an average boost of 2.0\% AUROC and 2.9\% inlier classification accuracy compared to the previous state-of-the-art method. (4) Our method serves as a versatile add-on for mainstream long-tailed learning methods, significantly enhancing their performance in detecting OOD samples. Importantly, our findings challenge the notion that a strong inlier classifier necessarily implies good OOD detection performance.

\section{Related Work}
\label{sec:related-works}

\paragraph{OOD detection}
 
As a representative approach, Outlier Exposure (OE) proposes maximizing the OOD data's predictive uncertainty as a complementary objective for the in-distribution classification loss. Further, Energy \cite{liu2020energy} improves OE by introducing the energy function as a regularization term and detects OOD samples according to their energy scores. 
Conversely, SOFL \cite{mohseni2020self} and ATOM \cite{chen2021atom} attempt to classify OOD samples into abstention classes while in-distribution samples are classified into their true classes. Then, OOD data can be identified according to the model's outputs on abstention classes.
Although existing OOD detectors can achieve high performance, they are typically trained on class-balanced in-distribution datasets and cannot be directly applied to long-tailed tasks.

\paragraph{Long-tailed learning}
Existing approaches to long-tailed learning can be roughly categorized into three types by modifying: (1) the inputs to a model by re-balancing the training data \cite{he2009learning,liu2019large,zhou2020bbn}; (2) the outputs of a model, for example by posthoc adjustment of the classifier \cite{kang2019decoupling,menon2020long,tang2020long} and (3) the internals of a model by modifying the loss function \cite{cui2019class,cao2019learning,jamal2020rethinking,Ren2020}. Recently, \cite{DBLP:conf/nips/YangX20} and \cite{DBLP:conf/icml/WeiTXF022} propose using OOD data to improve the performance of long-tailed learning. However, it is noted that these approaches are designed to boost the in-distribution classification performance and cannot be directly employed to detect OOD data.

\paragraph{Long-tailed OOD detection}
Recently, long-tailed OOD detection has received more and more attention, and several approaches have been proposed to tackle this challenging problem.
PASCL \cite{wang2022pascl} optimizes a contrastive objective between tail class samples and OOD data to push each other away in the latent representation space, which can boost the performance of OOD detection. Further, it minimizes the logit adjustment loss to yield a class-balanced performance of inlier classification.
HOD \cite{roy2022does} studies a long-tail OOD detection problem in medical image analysis, which directly trains a binary classifier to discriminate in-distribution data and OOD data. However, HOD assumes that the OOD data is labeled, while we do not make this assumption and only leverage unlabeled OOD data to aid the detection performance.
OLTR \cite{liu2019large} formally studies the OOD detection task in long-tailed learning. It detects OOD inputs in the latent representation space according to the minimum distance between them and the centroids of in-distribution classes. Although OLTR outperforms several OOD detectors such as MSP \cite{hendrycks17baseline}, it is outperformed by the state-of-the-art OOD detection methods, suggesting that there remains room for improvement.

\section{The Proposed Approach}
\label{sec:solutions}
\subsection{Overview}
We follow the popular training objective of existing state-of-the-art OOD detection methods, which train the model using both in-distribution data and unlabeled OOD data. 
Let $ \din $ and $\dout$  denote an in-distribution training set and an unlabeled OOD training set, respectively. Note that $\din $ follows a long-tailed class distribution in our setup. The training loss function of many existing OOD detection methods (e.g., OE, EnergyOE, ATOM, and PASCL) is defined as follows:
\begin{equation}
    \mathcal{L}_{\text{total}} = 
    \Lin + 
    \lambda \cdot \Lout,
\label{eq:overall-loss}
\end{equation}
where $\Lin$ is the inlier classification loss, $\Lout$ is the outlier detection loss, and $\lambda$ is a trade-off hyperparameter.
Typically, we choose to optimize the standard cross entropy loss (denoted by $\ell$) for the inlier classification task:
\begin{align}
\Lin&=\mathbb{E}_{\x\sim\din}[\ell(f(\boldsymbol{x}), y)] \nonumber\\
&= \log[1 +  \sum_{\mathit{y'} \neq \mathit{y} } \mathit{e}^{({\mathit{f}_{\mathit{y'}} {(\x)} - \mathit{f}_{\mathit{y}} {(\x)} )} }]
\label{eq:ce-id}
\end{align}
Here, $f_{y}(\x)$ represents the predicted logit corresponding to label $y$.
For OOD detection, we propose using $k$ abstention classes. The training outlier data is assigned to abstention classes by generating ``virtual'' labels by the model, and virtual labels may change through training iterations. With this, the training objective for outliers is defined as:
\begin{align}
\Lout&=\mathbb{E}_{\boldsymbol{\widetilde{x}}\sim\dout}[\ell(f(\boldsymbol{\widetilde{x}}), \widetilde{y})] \nonumber \\
&=\log[1 +  \sum_{\mathit{y'} \neq \widetilde{y} } \mathit{e}^{({\mathit{f}_{\mathit{y'}} {(\boldsymbol{\widetilde{x}})} - \mathit{f}_{\widetilde{y}} {(\boldsymbol{\widetilde{x}})} )} }] \nonumber\\
&\text{s.t.}\;\;\;  \widetilde{y} = \arg \max_{c\in [C+1,C+k]} f_c(\widetilde{\boldsymbol{x}})
\label{eq:ce-ood}
\end{align}
where $\widetilde{y}$ is the virtual label of outlier sample $\widetilde{\boldsymbol{x}}$. Note that our treatment for training outlier data differs from existing methods, including OE, Energy, and PASCL. They attempt to maximize the predictive uncertainties of outliers. We demonstrate that our approach achieves significantly better results in the experiments by introducing multiple abstention classes. The proposed approach is detailed below.

\subsection{OOD Samples with Dynamic Virtual Labels}
The approach of using abstention OOD classes is motivated by recent works \cite{abdelzad2019detecting,chen2021atom,vernekar2019out} which propose to add a single abstention class for all outlier data. Although this is shown to be effective compared to the outlier exposure method \cite{hendrycks2018deep}, fitting a heterogeneous outlier set to a single class is challenging and problematic.
One natural mitigation strategy here is to assign multiple abstention classes as possible outputs, which essentially turns the $C$-class classification into a $(C+k)$-class classification problem. Here, we denote $C$ as the number of inlier classes and $k$ as the number of abstention classes added for outliers. 
Taking CIFAR100-LT as an example, if we use an additional $k=30$ classes for fitting outliers, the number of neurons in the final fully-connected layer will be $130$. 

Ultimately, we want our model to classify unseen outliers in the test set into those $k$ abstention classes. This can be achieved by encouraging the model to learn a structured decision boundary for the inliers \textit{vs.} outliers. However, the ground-truth labels for training outlier data are not accessible. Thus, we propose generating virtual labels for the outliers so that the model learns to distinguish them from inliers. Towards this end, we take the predictions of the immediate model as the virtual labels at each training iteration, also known as self-labelling. The model is trained to predict virtual labels by minimizing the cross-entropy loss at the next iteration. The generation of virtual labels coincides with the self-training process, which is a popular framework in semi-supervised learning.
At test time, the sum of probabilities for the $k$ abstention classes indicating the OOD score is used. This is because the abstention classes are meaningless and virtual labels do not correspond to their ground-truth labels.

\paragraph{Mathmatical Interpretation} Exploring the reason behind OOD samples yielding higher scores than in-distribution samples is an intriguing endeavor. One way to comprehend the impact of virtual labels is through the lens of noise in loss gradients \cite{wei2021odnl}. We define the trainable parameter of model $f$ as $\boldsymbol{\theta} \in \mathbb{R}^{p}$. By calculating the gradient of the loss function with respect to $\boldsymbol{\theta}$ and updating the parameter accordingly, we gain insight into this phenomenon. Specifically, we represent the output probabilities for an in-distribution sample $\boldsymbol{x}$ and an OOD sample $\boldsymbol{\widetilde{x}}$ as $\boldsymbol{z}=\text{Softmax}(f(\boldsymbol{x}))$ and $\boldsymbol{\widetilde{z}}=\text{Softmax}(f(\boldsymbol{\widetilde{x}}))$ respectively.
%
\begin{prop}
\label{prop:gradient}
For the cross-entropy loss, Eq.~\eqref{eq:ce-ood} induces gradient noise $\boldsymbol{g} = - \frac{\nabla_{\boldsymbol{\theta}}\boldsymbol{\widetilde{z}}_j}{\boldsymbol{\widetilde{z}}_j}$ on $\nabla_{\boldsymbol{\theta}}\ell(\boldsymbol{z},y)$, s.t., $\boldsymbol{g} \in \mathbb{R}^{p}, j = \arg \max_{j\in [C+1,C+k]} \boldsymbol{\widetilde{z}}$. While each OOD sample in OE \cite{hendrycks2018deep} induces gradient noise $\boldsymbol{g}^\prime = -\frac{1}{C}\sum_{j=1}^{C}\frac{\nabla_{\boldsymbol{\theta}}\boldsymbol{\widetilde{z}}_j}{\boldsymbol{\widetilde{z}}_j}$ on $\nabla_{\boldsymbol{\theta}}\ell(\boldsymbol{z},y)$, where $\frac{\cdot}{\boldsymbol{\widetilde{z}}_j}$ denotes the element-wise division.
\end{prop}

\noindent
\textit{Remark.} The detailed proof for the following proposition can be found in the supplementary material. We first draw the conclusion that our proposed virtual labeling induces gradient noise of $\boldsymbol{g} = - \frac{\nabla_{\boldsymbol{\theta}}\boldsymbol{\widetilde{z}}_j}{\boldsymbol{\widetilde{z}}j}$ where $j \in [C+1, C+k]$ is the virtual label for an OOD sample. On the contrary, previous method OE induces gradient noise of $\boldsymbol{g}^\prime=-\frac{1}{C}\sum_{i=1}^{C}\frac{\nabla_{\boldsymbol{\theta}}\boldsymbol{\widetilde{z}}_i}{\boldsymbol{\widetilde{z}}i}$. Therefore, the main advantage of our approach yields gradient noise with dynamic direction depending on the virtual label of each OOD sample, which helps escape local minima during optimization. However, OE induces constant gradient noise so that the optimization of the model always follows the direction of gradient descent. Furthermore, the gradient noise induced by our approach helps the model to produce more conservative in-distribution class (i.e., $[1, C]$) predictions on OOD samples than OE. This is because of the nature of virtual labels which encourages the model to produce confident predictions on virtual classes, i.e., $[C+1, C+k]$. 

\paragraph{Context-rich Tail Class Augmentation}
In our pursuit to enhance generalization, we go beyond the utilization of virtual labels to amplify the distinction between in-distribution and OOD samples. We additionally harness OOD samples to augment the tail classes, leading to improved performance. Our approach involves the implementation of an image-mixing data augmentation technique called CutMix \cite{DBLP:conf/iccv/YunHCOYC19}, which enables us to generate training samples specifically tailored for the tail class. The core concept revolves around leveraging the context-rich nature of the head class and outlier images as backgrounds to create diverse and enriched tail samples.
Given a tail-class image $\boldsymbol{x}_f$, we combine it with a randomly selected head-class or OOD image represented as $\boldsymbol{x}_b$. This merging operation is referred to as the CutMix operator and is defined as follows:
\begin{equation}
\x^{\text{b}\odot\text{f}}=\mathbf M \odot \x^{\text{b}}+(\mathbf 1 - \mathbf M)  \odot  \x^{\text{f}}
\label{eq:cutmix}
\end{equation}
In this context, we designate $\boldsymbol{x}^{\text{b}}$ as the background image and $\boldsymbol{x}^{\text{f}}$ as the foreground image. A binary mask $\mathbf{M} \in {0, 1}^{W \times H}$ is employed to indicate the areas to preserve as background. Correspondingly, $(\mathbf{1} - \mathbf{M})$ selects the patch from the foreground image to be pasted onto the background image. Here, $\mathbf{1}$ represents a matrix filled with ones, and $\odot$ denotes element-wise multiplication. In order to address the limited availability of data for tail classes, we assume that the composite image $\boldsymbol{x}^{\text{b}} \odot \boldsymbol{x}^{\text{f}}$ carries the same label as the foreground image $\boldsymbol{x}^{\text{f}}$, i.e., $y^{\text{b} \odot \text{f}} = y^{\text{f}}$. However, it is important to note that this approach can introduce label noise during training. Therefore, we assign lower sample weights to the generated tail-class images to mitigate the adverse impact.

Tailored CutMix offers two notable advantages for both outlier detection and inlier classification. Firstly, by using diverse OOD images as backgrounds, the model is encouraged to differentiate between tail-class images and OOD data based on foreground objects rather than image backgrounds. This aids in enhancing the model's ability to identify and distinguish outliers effectively. Secondly, the inclusion of head-class and OOD data through mixing increases the frequency of tail classes, leading to a more balanced training set. This improved class balance contributes to enhanced generalization capabilities.

It is worth noting that the study conducted by \cite{park2021cmo} also incorporates CutMix to generate tail-class samples. However, their approach differs in that they sample image pairs from the original long-tailed data distribution and a tail-class-weighted distribution. Furthermore, their study focuses on improving inlier prediction accuracy rather than OOD detection. As far as our knowledge extends, we are the first to adapt CutMix specifically for long-tailed OOD detection, distinguishing our work in this area.

\subsection{Improving OOD Separation}
To amplify both outlier detection and inlier classification performance, we employ a mixture of experts by integrating multiple classifiers that share a common feature extractor. By training an ensemble of $m$ members with random initializations, we optimize the sum of loss functions for these classifiers, aiming to achieve superior results.
\begin{equation}
    \mathcal{L}_{\text{total}} = 
    \sum_{i=1}^m(\Lin^{(i)} + \lambda \cdot \Lout^{(i)}) 
\label{eq:all-loss}
\end{equation}
Given an input $\x$ at test time, we use the average predictions of ensemble members as the OOD score:
\begin{equation}
    G(\x) = \frac{1}{m} \sum_{i=1}^m \sum_{j=C+1}^{C+k} z_j^{(i)},
\end{equation}
where $  \boldsymbol{z} ^{(i)}= \text{Softmax} (f^{(i)}(\x))$. We choose $ m=3$ in our experiments. If $ \x$ is not deemed as an OOD input, the prediction will be $ \arg\max_{1\leq c \leq C} \frac{1}{m} \sum_{i=1}^m \boldsymbol{z}^{(i)}$.

\begin{figure}
    \centering
\includegraphics[width=0.95\linewidth]{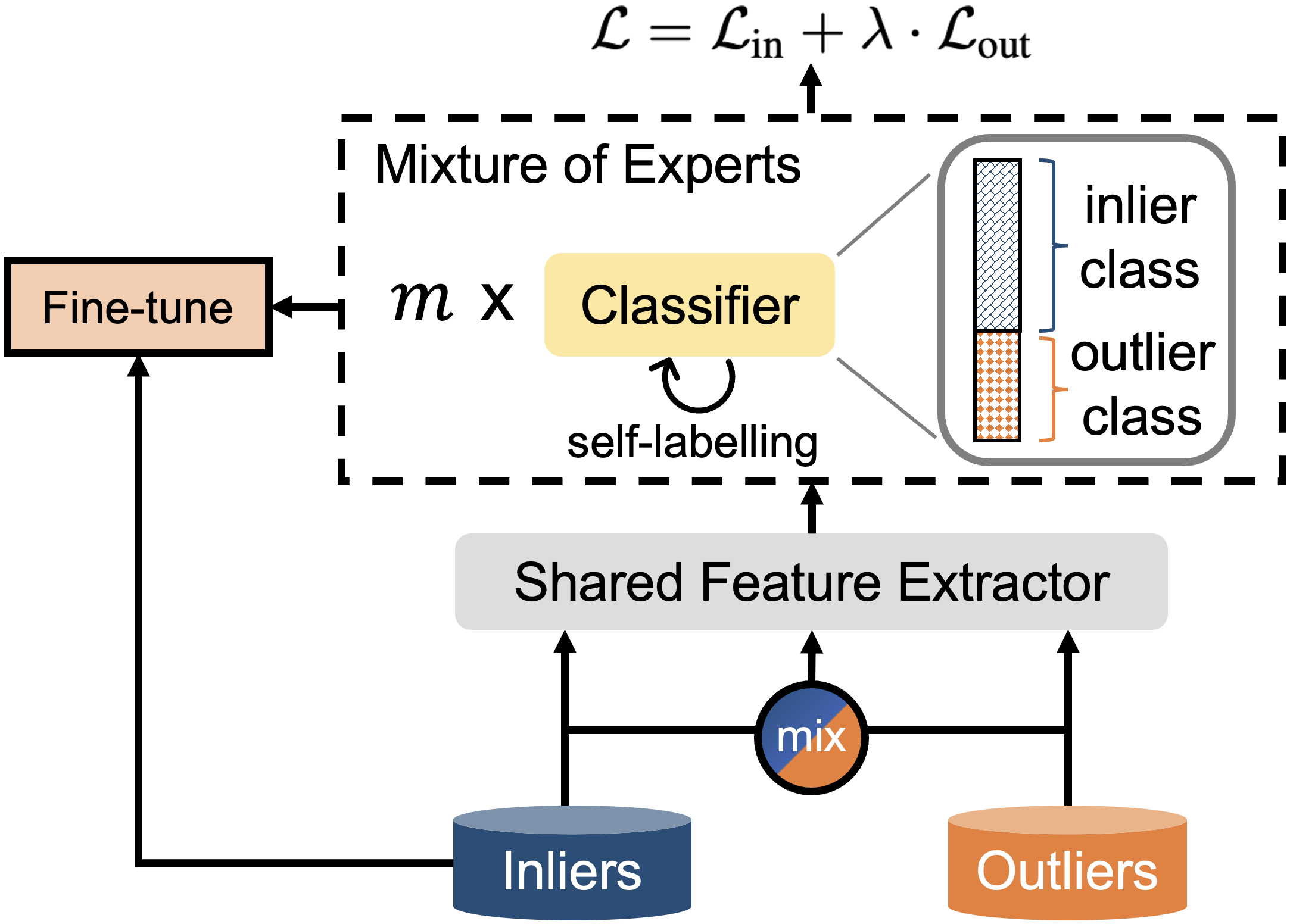}
\caption{Overview of EAT framework.}\label{fig:EAT-framework}
\end{figure}

It is important to note that the performance improvement achieved by deep ensembles relies on the diversity introduced through random initialization of network parameters. In our specific setup, since we employ a shared feature extractor, random initialization is applied solely to the parameters of the last layer. To enhance diversity further, we train ensemble models with virtual labels generated by each classifier. This means that the virtual label for a given sample $\boldsymbol{x}$ is obtained by selecting $\tilde{y} = \arg \max_{c \in [C+1,C+k]} f^{(i)}_c(\boldsymbol{x})$ for the $i$-th classifier. The overall framework of our approach is depicted in Figure \ref{fig:EAT-framework}.
%
%
%


\subsection{Model Fine-tuning }

After training a multi-branch model, we proceed to fine-tune the classifiers using training inlier data, aiming to enhance the classification performance. It is worth mentioning that the two-stage approach, involving representation learning followed by classifier learning, is commonly employed in the field of long-tailed learning. Prominent examples include decoupling \cite{kang2019decoupling}, BBN \cite{zhou2020bbn}, and MisLAS \cite{DBLP:conf/cvpr/ZhongC0J21}. In this article, we implicitly explore the advantages of utilizing OOD data for representation learning by optimizing supervised objectives during model training. Furthermore, we keep the feature extractor fixed and refine the classifiers for a few iterations to improve inlier classification performance. During the fine-tuning process, we employ the logits adjustment (LA) loss \cite{menon2020long} to guide the training.
\begin{equation}
\ell_{\text{LA}}(y,\mathit{f(\x)}) 
= \log[1 + \begin{matrix} \sum_{\mathit{y'} \neq \mathit{y} } \mathit{e}^{{\bigtriangleup}_{\mathit{y}\mathit{y'}}} \cdot \mathit{e}^{({\mathit{f}_{\mathit{y'}} {(\x)} - \mathit{f}_{\mathit{y}} {(\x)} )} }\end{matrix}]
\label{eq:stage2-ce-loss}
\end{equation}
Here, the pairwise label margins ${{\bigtriangleup}_{\mathit{y}\mathit{y'}}} = \log{\frac{\pi_{\mathit{y}}}{\pi_{\mathit{y'}}}}$ represents the desired gap between predictive confidence for $\mathit{y}$ and $\mathit{y'}$ depending on the number of each class. $\pi_y$ denotes the class prior of class $y$ in the training inlier data. We will empirically show that fine-tuning can introduce not only large improvements for the inlier classification task but also boost the performance of OOD detection.

\begin{table}[!h]
\centering
\footnotesize

\setlength{\tabcolsep}{0.5mm}
\begin{subtable}[h]{1.0\linewidth}

\begin{tabularx}{\linewidth}{ cccccc }
\toprule
$\douttest$ & \textbf{Method} & \textbf{AUROC} & \textbf{AUPR} & \textbf{FPR95} & \textbf{ACC95} \\
\midrule
\multirow{3}{*}{Texture} 
& OE & 92.59$_{\pm\text{0.42}}$ & 83.32$_{\pm \text{1.67}}$ &	25.10$_{\pm \text{1.08}}$ & 84.52$_{\pm \text{0.76}}$ \\
& PASCL & 93.16$_{\pm \text{0.37}}$ & {84.80}$_{\pm \text{1.50}}$ &{23.26}$_{\pm \text{0.91}}$ & {85.86}$_{\pm \text{0.72}}$\\
& \ours & \textbf{95.44}$_{\pm \text{0.46}}$ & \textbf{92.28}$_{\pm \text{0.95}}$ & \textbf{21.50}$_{\pm \text{1.50}}$ & \textbf{87.00}$_{\pm \text{0.66}}$\\
\midrule
\multirow{3}{*}{SVHN} 
& OE & 95.10$_{\pm \text{1.01}}$ & 97.14$_{\pm \text{0.81}}$ & 16.15$_{\pm \text{1.52}}$ & 81.33$_{\pm \text{0.81}}$ \\
& PASCL & {96.63}$_{\pm \text{0.90}}$ & {98.06}$_{\pm \text{0.56}}$ & {12.18}$_{\pm \text{3.33}}$ & {82.72}$_{\pm \text{1.51}}$\\
& \ours & \textbf{97.92}$_{\pm \text{0.36}}$ & \textbf{99.06}$_{\pm \text{0.20}}$ & \textbf{9.87}$_{\pm \text{2.06}}$ & \textbf{84.39}$_{\pm \text{0.51}}$\\
\midrule
\multirow{3}{*}{CIFAR100} 
& OE & 83.40$_{\pm	\text{0.30}}$ & 80.93$_{\pm \text{0.57}}$ & 56.96$_{\pm 0.91}$ & {94.56}$_{\pm \text{0.57}}$ \\
& PASCL & {84.43}$_{\pm \text{0.23}}$ & {82.99}$_{\pm \text{0.48}}$ & 57.27$_{\pm \text{0.88}}$ & 94.48$_{\pm \text{0.31}}$ \\
& \ours & \textbf{85.93}$_{\pm \text{0.15}}$ & \textbf{86.10}$_{\pm \text{0.35}}$ & \textbf{54.13}$_{\pm \text{0.63}}$ & \textbf{95.81}$_{\pm \text{0.41}}$ \\
\midrule
\multirow{3}{*}{\makecell{Tiny\\ImageNet}} 
& OE & 86.14$_{\pm \text{0.29}}$ & 79.33$_{\pm \text{0.65}}$ & 47.78$_{\pm \text{0.72}}$ & 91.19$_{\pm \text{0.33}}$ \\
& PASCL & {87.14}$_{\pm \text{0.18}}$ & {81.54}$_{\pm \text{0.38}}$ & {47.69}$_{\pm \text{0.59}}$ & {91.20}$_{\pm \text{0.35}}$ \\
& \ours & \textbf{89.11}$_{\pm \text{0.34}}$ & \textbf{85.43}$_{\pm \text{0.58}}$ & \textbf{41.75}$_{\pm \text{0.68}}$ & \textbf{91.67}$_{\pm \text{0.65}}$ \\
\midrule
\multirow{3}{*}{LSUN} 
& OE & 91.35$_{\pm	\text{0.23}}$ & 87.62$_{\pm \text{0.82}}$ & 27.86$_{\pm \text{0.68}}$ & 85.49$_{\pm \text{0.69}}$ \\
& PASCL & {93.17}$_{\pm \text{0.15}}$ & {91.76}$_{\pm \text{0.53}}$ & {26.40}$_{\pm \text{1.00}}$ & {86.67}$_{\pm \text{0.90}}$ \\
& \ours & \textbf{95.13}$_{\pm \text{0.43}}$ & \textbf{94.12}$_{\pm \text{0.61}}$ & \textbf{19.72}$_{\pm \text{1.61}}$ & \textbf{86.68}$_{\pm \text{0.64}}$ \\
\midrule
\multirow{3}{*}{Places365} 
& OE & 90.07$_{\pm \text{0.26}}$ & 95.15$_{\pm \text{0.24}}$ & 34.04$_{\pm \text{0.91}}$ & 87.07$_{\pm \text{0.53}}$ \\
& PASCL & {91.43}$_{\pm \text{0.17}}$ & {96.28}$_{\pm \text{0.14}}$ & {33.40}$_{\pm \text{0.88}}$ & \textbf{87.87}$_{\pm \text{0.71}}$ \\
& \ours & \textbf{93.68}$_{\pm \text{0.27}}$ & \textbf{97.42}$_{\pm \text{0.14}}$ & \textbf{26.03}$_{\pm \text{0.92}}$ & {87.64}$_{\pm \text{0.68}}$ \\
\midrule
\multirow{3}{*}{Average} 
& OE & 89.77$_{\pm \text{0.27}}$ & 87.25$_{\pm \text{0.61}}$ & 34.65$_{\pm \text{0.46}}$ & 87.36$_{\pm \text{0.51}}$ \\
& PASCL &  {90.99}$_{\pm  \text{0.19}}$ & {89.24}$_{\pm \text{0.34}}$ & {33.36}$_{\pm \text{0.79}}$ & {88.13}$_{\pm \text{0.56}}$\\
& \ours &  \textbf{92.87}$_{\pm  \text{0.33}}$ & \textbf{92.40}$_{\pm \text{0.47}}$ & \textbf{28.83}$_{\pm \text{1.23}}$ & \textbf{88.86}$_{\pm \text{0.59}}$\\
\bottomrule
\end{tabularx}
\caption{OOD detection results and in-distribution classification results in terms of ACC95.}
\end{subtable}

\begin{subtable}{\linewidth}
 \centering
\begin{tabular}{ ccccc }
\toprule
\multirow{2}{*}{\textbf{Method}}  & \multicolumn{4}{c}{\textbf{ACC@FPRn ($\uparrow$)}}   \\
& 0 & 0.001 & 0.01 & 0.1 \\
\midrule
OE & 73.54$_{\pm \text{0.77}}$ & 73.90$_{\pm \text{0.77}}$ & 74.46$_{\pm \text{0.81}}$ & 78.88$_{\pm \text{0.66}}$\\
PASCL & 77.08$_{\pm \text{1.01}}$ & 77.13$_{\pm \text{1.02}}$ & 77.64$_{\pm \text{0.99}}$ & 81.96$_{\pm \text{0.85}}$ \\
\ours & \textbf{81.31}$_{\pm \text{0.26}}$& \textbf{81.36}$_{\pm \text{0.25}}$ & \textbf{81.81}$_{\pm \text{0.26}}$ & \textbf{84.40}$_{\pm \text{0.28}}$ \\
\bottomrule
\end{tabular}
\caption{In-distribution classification results in terms of ACC@FPR$n$.}
\end{subtable}

\begin{subtable}{\linewidth}
 \centering
\begin{tabular}{ ccccc }
\toprule
\textbf{Method} & \textbf{AUROC ($\uparrow$)} & \textbf{AUPR ($\uparrow$)} & \textbf{FPR95 ($\downarrow$)} & \textbf{ACC ($\uparrow$)} \\
\midrule
ST (MSP) & 72.28 & 70.27 & 66.07 & 72.34\\
 OECC & 87.28 & 86.29 & 45.24 & 60.16 \\
 EnergyOE & 89.31 & {88.92} & 40.88 & {74.68} \\
 OE & 89.77$_{\pm \text{0.27}}$ & 87.25$_{\pm \text{0.61}}$ & 34.65$_{\pm \text{0.46}}$ & 73.84$_{\pm \text{0.77}}$\\
 PASCL &  \underline{90.99}$_{\pm \text{0.19}}$& \underline{89.24}$_{\pm \text{0.34}}$& \underline{33.36}$_{\pm \text{0.79}}$& \underline{77.08}$_{\pm \text{1.01}}$\\
 \ours &  \textbf{92.87}$_{\pm \text{0.33}}$& \textbf{92.40}$_{\pm \text{0.47}}$& \textbf{28.83}$_{\pm \text{1.23}}$& \textbf{81.31}$_{\pm \text{0.26}}$\\
\bottomrule
\end{tabular}
\caption{Comparison with other methods.}
\vspace{-0.8em}

\end{subtable}

\caption{Results on CIFAR10-LT using ResNet18. The best results are shown in bold. Mean and standard deviation over six random runs are reported.
``Average'' means the results averaged across six different $\douttest$ sets.
}\label{tab:cifar10-lt-0.01-ResNet18}
\vspace{-2em}
\end{table}

\section{Experiments}
\subsection{Experiment Settings}
\label{sec:settings}
We verify our approach on commonly used datasets in comparison with the existing state-of-the-art. CIFAR10-LT, CIFAR100-LT \cite{cao2019learning}, and ImageNet-LT \cite{liu2019large} are used as in-distribution training sets ($\din$). The standard CIFAR10, CIFAR100, and ImageNet test sets are used as in-distribution test sets ($\dintest$). Following \cite{wang2022pascl}, we set the default imbalance ratio to 100 for CIFAR10-LT and CIFAR100-LT during training.

\paragraph{OOD datasets for CIFAR-LT}
We employ $300$ thousand samples from TinyImages80M \cite{torralba200880} as the OOD training images for CIFAT10-LT and CIFAR100-LT following \cite{hendrycks2018deep,wang2022pascl}. Of those, 80 Million Tiny Images is a large-scale, diverse dataset of ${32}\times{32}$ natural images. The 300 thousand samples are selected from the  80 Million Tiny Images by \cite{hendrycks2018deep}, not intersected with the CIFAR datasets. For OOD test data, we use Textures \cite{cimpoi2014describing}, SVHN \cite{netzer2011reading}, Tiny ImageNet \cite{letiny}, LSUN \cite{yu2015lsun}, and Places365 \cite{zhou2017places} as $\douttest$. We use CIFAR-100 as a $\douttest$ for CIFAR10-LT and vice-versa. 
 \paragraph{OOD datasets for ImageNet-LT}
We use a specifically designed $\dout$ called ImageNet-Extra following \cite{wang2022pascl}. ImageNet-Extra has $517,711$ images belonging to $500$ classes randomly sampled from ImageNet-22k \cite{deng2009imagenet}, but not overlapping with the $1,000$ in-distribution classes in ImageNet-LT.
For $\douttest$, we use ImageNet-1k-OOD constructed by \cite{wang2022pascl}, which has $50,000$ OOD test images from $1,000$ classes randomly selected from ImageNet-22k (with $50$ images in each class). Considering the fairness of OOD detection, it has the same size as the in-distribution test set. 
The $1,000$ classes in ImageNet-1k-OOD are not intersecting either the $1,000$ in-distribution classes in ImageNet-LT or the $500$ OOD training classes in ImageNet-Extra.
To ensure the rigor of the experiment, ImageNet-LT $\mathcal{D}_{{\text{train}}}^{\text{in}}$, ImageNet-Extra $\mathcal{D}_{{\text{train}}}^{\text{out}}$, ImageNet-1k-OOD $\mathcal{D}_{{\text{test}}}^{\text{out}}$, and ImageNet 
$\mathcal{D}_{{\text{test}}}^{\text{in}}$ are orthogonal.

\paragraph{Evaluation measures}
Following \cite{hendrycks2018deep,mohseni2020self,yang2021semantically,wang2022pascl}, we use the below evaluation measures:
\begin{itemize}[leftmargin=*]
\item \textbf{AUROC} ($\uparrow$): The area under the receiver operating characteristic curve, which measures if the detector can rank OOD samples higher in-distribution samples.
\item \textbf{AUPR} ($\uparrow$): The area under the precision-recall curve, which means the average precision over all recall values.
\item \textbf{FPR@TPR$n$} ($\downarrow$): The false positive rate (FPR) when $n$ (in percentage) OOD samples have been successfully detected (i.e., when the true positive rate (TPR) is $n$).
\cite{hendrycks17baseline} have primarily used the measure FPR95, also known as FPR@TPR$95\%$. 
\item \textbf{ACC@TPR$n$} ($\uparrow$): The classification accuracy on the remaining in-distribution data when $n$ (in percentage) OOD samples have been successfully detected. The term ACC@TPR$95\%$ is shortened to ACC95. 
\item \textbf{ACC@FPR$n$} ($\uparrow$): The classification accuracy on the remaining in-distribution data when $n$ (in percentage) in-distribution samples are mistakenly detected as OOD. 
The accuracy on the overall in-distribution test set is known as ACC@FPR$0$, or simply ACC.
\end{itemize}

\paragraph{Model Configuration}
The current best long-tail OOD detection method is PASCL, and before that it is OE, so we mainly compare the experimental results with these two baseline methods.
For experiments on CIFAR10 and CIFAR100, we use the ResNet18 \cite{he2016deep} following \cite{yang2021semantically}. 
For experiments on CIFAR10-LT and CIFAR100-LT, we train the model for $180$ epochs using Adam \cite{kingma2014adam} optimizer with initial learning rate $1 \times 10^{-3}$ and batch size $128$. We decay the learning rate to $0$ using a cosine annealing learning rate scheduler \cite{loshchilov2016sgdr}. 
For fine-tuning, we fine-tune the classifier and BN layers for $10$ epochs using Adam optimizer with an initial learning rate $5\times 10^{-4}$. Other hyper-parameters are the same as in classifier layer fine-tuning. 
For experiments on ImageNet-LT, we follow the settings in \cite{wang2020long} and use ResNet50 \cite{he2016deep}. We train the main branch for $60$ epochs using SGD optimizer with an initial learning rate of $0.1$ and batch size of $64$. We fine-tune the classifier for the $1$ epoch using SGD optimizer with an initial learning rate of $0.01$.
In all experiments, we set $\lambda = 0.05$, and the weights for generated tail class samples are set to $0.05$ for EAT. For the number of abstention classes, we set $k=3$ on CIFAR10-LT, $k=30$ on CIFAR100-LT and ImageNet-LT.
For other hyper-parameters in the baseline methods, we use the suggested values in their original papers.

\begin{table}[!h]
\centering
\footnotesize

\setlength{\tabcolsep}{0.5mm}
\begin{subtable}{\linewidth}

\begin{tabular}{ cccccc }
\toprule
$\douttest$ & \textbf{Method} & \textbf{AUROC} & \textbf{AUPR} & \textbf{FPR95} & \textbf{ACC95} \\
\midrule
\multirow{2}{*}{Texture} 
& OE & 76.71$_{\pm \text{1.20}}$ & 58.79$_{\pm \text{1.39}}$ &	68.28$_{\pm \text{1.53}}$ & 	71.43$_{\pm \text{1.58}}$\\
& PASCL & 76.01$_{\pm \text{0.66}}$& 58.12$_{\pm \text{1.06}}$ & \textbf{67.43}$_{\pm \text{1.93}}$& 73.11$_{\pm \text{1.55}}$\\
& \ours & \textbf{80.27}$_{\pm \text{0.76}}$& \textbf{71.76}$_{\pm \text{1.56}}$& {67.53}$_{\pm \text{0.64}}$& \textbf{73.76}$_{\pm \text{0.75}}$\\
\midrule
\multirow{2}{*}{SVHN} 
& OE & 77.61$_{\pm \text{3.26}}$ & 86.82$_{\pm \text{2.50}}$& 58.04$_{\pm \text{4.82}}$&  64.27$_{\pm \text{3.26}}$\\
& PASCL & 80.19$_{\pm \text{2.19}}$& 88.49$_{\pm \text{1.59}}$ & 53.45$_{\pm \text{3.60}}$& \textbf{64.50}$_{\pm \text{1.87}}$\\
& \ours & \textbf{83.11}$_{\pm \text{2.83}}$& \textbf{89.71}$_{\pm \text{2.08}}$& \textbf{47.78}$_{\pm \text{4.87}}$& {61.67}$_{\pm \text{2.65}}$ \\
\midrule
\multirow{2}{*}{CIFAR10} 
& OE & 62.23$_{\pm \text{0.30}}$ & \textbf{57.57}$_{\pm \text{0.34}}$ & 80.64$_{\pm \text{0.98}}$ & \textbf{82.67}$_{\pm \text{0.99}}$ \\
& PASCL & \textbf{62.33}$_{\pm \text{0.38}}$ & 57.14$_{\pm \text{0.20}}$ & {79.55}$_{\pm \text{0.84}}$ & {82.30}$_{\pm \text{1.07}}$\\
& \ours & {61.62}$_{\pm \text{0.47}}$ & 55.30$_{\pm \text{0.54}}$ & \textbf{77.97}$_{\pm \text{0.77}}$ & {82.61}$_{\pm \text{0.61}}$\\
\midrule
\multirow{2}{*}{\makecell{Tiny\\ImageNet}} 
& OE & 68.04$_{\pm \text{0.37}}$ & 51.66$_{\pm \text{0.51}}$ & 76.66$_{\pm \text{0.47}}$ & 76.22$_{\pm \text{0.61}}$ \\
& PASCL & {68.20}$_{\pm \text{0.37}}$ & 51.53$_{\pm \text{0.42}}$ & {76.11}$_{\pm \text{0.80}}$ & \textbf{77.56}$_{\pm \text{1.15}}$\\
& \ours & \textbf{68.34}$_{\pm \text{0.28}}$ & \textbf{52.79}$_{\pm \text{0.25}}$ & \textbf{74.89}$_{\pm \text{0.49}}$ & {77.07}$_{\pm \text{0.39}}$\\
\midrule
\multirow{2}{*}{LSUN} 
& OE & 77.10$_{\pm \text{0.64}}$ & {61.42}$_{\pm \text{0.99}}$ & 63.98$_{\pm \text{1.38}}$ &	65.64$_{\pm \text{1.03}}$ \\
& PASCL & {77.19}$_{\pm \text{0.44}}$ & {61.27}$_{\pm \text{0.72}}$ & {63.31}$_{\pm \text{0.87}}$ & \textbf{68.05}$_{\pm \text{1.24}}$ \\
& \ours & \textbf{81.09}$_{\pm \text{0.32}}$ & \textbf{67.46}$_{\pm \text{0.64}}$ & \textbf{55.02}$_{\pm \text{1.20}}$ & {62.07}$_{\pm \text{0.78}}$ \\
\midrule
\multirow{2}{*}{Places365} 
& OE & 75.80$_{\pm \text{0.45}}$ &	86.68$_{\pm \text{0.38}}$ &	65.72$_{\pm \text{0.92}}$ &	67.04$_{\pm \text{0.49}}$ \\
& PASCL &  {76.02}$_{\pm \text{0.21}}$ & 86.52$_{\pm \text{0.29}}$ & {64.81}$_{\pm \text{0.27}}$ & \textbf{69.04}$_{\pm \text{0.90}}$ \\
& \ours &  \textbf{78.28}$_{\pm \text{0.31}}$ & \textbf{88.20}$_{\pm \text{0.20}}$ & \textbf{60.85}$_{\pm \text{0.69}}$ & {66.15}$_{\pm \text{0.68}}$ \\
\midrule
\multirow{2}{*}{Average} 
& OE & 72.91$_{\pm \text{0.68}}$ & 67.16$_{\pm \text{0.57}}$ & 68.89$_{\pm \text{1.07}}$ & 71.21$_{\pm \text{0.84}}$ \\
& PASCL & {73.32}$_{\pm \text{0.32}}$ & {67.18}$_{\pm \text{0.10}}$ & {67.44}$_{\pm \text{0.58}}$ & \textbf{72.43}$_{\pm \text{0.66}}$ \\
& \ours & \textbf{75.45}$_{\pm \text{0.83}}$ & \textbf{70.87}$_{\pm \text{0.88}}$ & \textbf{64.01}$_{\pm \text{1.44}}$ & {70.55}$_{\pm \text{0.98}}$ \\
\bottomrule
\end{tabular}
\caption{In-distribution classification results in terms of ACC@FPR$n$.}
\end{subtable}
\\
\begin{subtable}{\linewidth}
\centering
\begin{tabular}{ ccccc }
\toprule
\multirow{2}{*}{\textbf{Method}}  & \multicolumn{4}{c}{\textbf{ACC@FPRn ($\uparrow$)}}   \\
& 0 & 0.001 & 0.01 & 0.1 \\
\midrule
OE & 39.04$_{\pm \text{0.37}}$  & 39.07$_{\pm \text{0.38}}$ & 39.38$_{\pm \text{0.38}}$ & 42.40$_{\pm \text{0.44}}$ \\
PASCL & {43.10}$_{\pm \text{0.47}}$ & {43.12}$_{\pm \text{0.47}}$ & {43.39}$_{\pm \text{0.48}}$ & {46.14}$_{\pm \text{0.38}}$ \\
\ours & \textbf{46.23}$_{\pm \text{0.25}}$ & \textbf{46.24}$_{\pm \text{0.25}}$ & \textbf{46.38}$_{\pm \text{0.23}}$ & \textbf{48.39}$_{\pm \text{0.32}}$ \\
\bottomrule
\end{tabular}
\caption{OOD detection results and in-distribution classification results in terms of ACC95.}
\end{subtable}
\\
\begin{subtable}{\linewidth}
\centering
\begin{tabular}{ ccccc }
\toprule
 \textbf{Method} & \textbf{AUROC ($\uparrow$)} & \textbf{AUPR ($\uparrow$)} & \textbf{FPR95 ($\downarrow$)} & \textbf{ACC ($\uparrow$)} \\
\midrule
ST (MSP) & 61.00 & 57.54 & 82.01 & {40.97}  \\
OECC &  70.38 & 66.87 & 73.15 & 32.93 \\
EnergyOE & 71.10 & {67.23} & 71.78 & 39.05  \\
OE & {72.91}$_{\pm \text{0.68}}$ & 67.16$_{\pm \text{0.57}}$ & {68.89}$_{\pm \text{1.07}}$ & 39.04$_{\pm \text{0.37}}$  \\
PASCL & \underline{73.32}$_{\pm \text{0.32}}$ & \underline{67.18}$_{\pm \text{0.10}}$ & \underline{67.44}$_{\pm \text{0.58}}$ & \underline{43.10}$_{\pm \text{0.47}}$ \\
\ours & \textbf{75.45}$_{\pm \text{0.83}}$ & \textbf{70.87}$_{\pm \text{0.88}}$ & \textbf{64.01}$_{\pm \text{1.44}}$ & \textbf{46.23}$_{\pm \text{0.25}}$ \\
\bottomrule
\end{tabular}
\caption{Comparison with other methods.}
\end{subtable}
\vspace{-2em}
\caption{Results on CIFAR100-LT using ResNet18. The best results are shown in bold. Mean and standard deviation over six random runs are reported.
``Average'' means the results averaged across six different $\douttest$ sets.}
\label{tab:cifar100-lt-0.01-ResNet18}
\vspace{-2em}
\end{table}

\subsection{Main Results}
\label{sec:results}
\Cref{tab:cifar10-lt-0.01-ResNet18}, \Cref{tab:cifar100-lt-0.01-ResNet18}, and \Cref{tab:imagenet-lt-ResNet50} report the results for CIFAR10-LT, CIFAR100-LT, and ImageNet-LT datasets, respectively. For fair comparisons, results of existing methods are directly borrowed from \cite{wang2022pascl}.
There are three sub-tables in \Cref{tab:cifar10-lt-0.01-ResNet18} and \Cref{tab:cifar100-lt-0.01-ResNet18}: since performance measures may differ across different $\douttest$ datasets, we report AUROC, AUPR, FPR95, and ACC95 on each $\douttest$ as well as the average values across six $\douttest$ datasets in sub-table (a).
In sub-table (b), we compare ACC@FPR$n$ with various $n$ values that are independent of $\douttest$. 
Finally, we put together four main performance measures in terms of both outlier detection and inlier classification in sub-table (c). 

From the results, we can see that EAT significantly outperforms OE, PASCL, and other baselines.
For instance, on CIFAR10-LT, EAT achieves $1.88\%$ AUROC, $3.16\%$ AUPR, $4.53\%$  FPR95, $0.73\%$ ACC95, and $4.23\%$ in-distribution accuracy improvement than PASCL on average. Likewise, on CIFAR100-LT, our approach achieves $2.13\%$ higher AUROC, $3.69\%$ higher AUPR, $3.43\%$ lower FPR95, and $3.13\%$ higher classification accuracy than PASCL on average.

\begin{table}[!h]
\centering
\begin{tabular}{ ccc }
\toprule
\multirow{2}{*}{\textbf{Method}} & \multicolumn{2}{c}{\textbf{ACC ($\uparrow$)}} \\
&  Head classes & Tail classes \\
\midrule
OE & 54.29 & 20.90 \\
PASCL & 54.73 (+0.44) & 36.26 (+15.36) \\
\ours & 59.46 (+5.17) & 34.12 (+13.22) \\
\bottomrule
\end{tabular}
\caption{Results on ImageNet-LT. }
\label{tab:classwise}
\vspace{-2em}
\end{table}
%
On ImageNet-LT, our approach achieves the best results in seven cases, while the previous state-of-the-art method PASCL performs the best in only one case. Compared with OE, our approach achieves $3.51\%$ higher AUROC, $0.96\%$ higher AUPR, and $9.19\%$ higher in-distribution accuracy.

\begin{table*}[htbp]
\centering
\resizebox{\linewidth}{!}{
\begin{tabular}{ c|cc|cccc|cccc|cccc }
\toprule
\multirow{2}{*}{\textbf{Method}} & \multirow{2}{*}{\textbf{AUROC ($\uparrow$)}} & \multirow{2}{*}{\textbf{AUPR ($\uparrow$)}} & \multicolumn{4}{c|}{\textbf{FPR@TPR$n$ ($\downarrow$)}} & \multicolumn{4}{c|}{\textbf{ACC@TPR$n$ ($\uparrow$)}} & \multicolumn{4}{c}{\textbf{ACC@FPR$n$ ($\uparrow$)}}  \\
& & & 0.98 & 0.95 & 0.90 & 0.80 & 0.98 & 0.95 & 0.90 & 0.80 & 0 & 0.001 & 0.01 & 0.1 \\
\midrule
ST (MSP) & 53.81	& 51.63	&95.38&	90.15&	83.52&	72.97	&\textbf{96.67}&	\textbf{92.61}&	\textbf{87.43}&	\textbf{77.52}&	{39.65} &	{39.68} &	{40.00} &	{43.18} \\
OECC & 63.07 & 63.05 & \textbf{93.15}  & \textbf{86.90} & 78.79 & 65.23 & 94.25 & 88.23 & 80.12 & 68.36 & 38.25 & 38.28 & 38.56 & 41.47 \\
EnergyOE & 64.76&	64.77&	\underline{94.15}&	87.72 &	{78.36} & {63.71} &	80.18&	74.38&	67.65&	59.68&	38.50&	38.52&	38.72&	40.99 \\
OE & {66.33}&	{68.29}&	95.11&	88.22&	78.68&	65.28&	95.46&	88.22&	78.68&	65.28&	37.60&	37.62&	37.79&	40.00 \\
PASCL & \underline{68.00}&	\textbf{70.15}&	94.38 &	\underline{87.53}&	\underline{78.12}&	\underline{62.48} & \underline{95.69} & \underline{89.55} &	\underline{80.88} &	\underline{69.60} &	\underline{45.49} & \underline{45.51} & \underline{45.62} & \underline{47.49} \\
\ours & \textbf{69.84}&	\underline{69.25} & {94.34}&	{87.63}&	\textbf{77.30} & \textbf{57.81} & {83.22} &	{77.80} &{70.84} &	{61.49}&	\textbf{46.79}&	\textbf{46.79}&	\textbf{46.83} & \textbf{48.30} \\
\bottomrule
\end{tabular}
}
\caption{Results on ImageNet-LT. The best and second-best are bolded and underlined, respectively.}
\label{tab:imagenet-lt-ResNet50}
\vspace{-0.5em}
\end{table*}

\begin{table*}[!h]
\centering

\resizebox{1\linewidth}{!}{
\begin{tabular}{c|cccc|cccccccc }
\toprule
\multirow{2}{*}{$\din$} & \multirow{2}{*}{Virtual label} & \multirow{2}{*}{Fine-tuning} 
 & \multirow{2}{*}{CutMix} & \multirow{2}{*}{MoE} & \multirow{2}{*}{\textbf{AUROC ($\uparrow$)}} & \multirow{2}{*}{\textbf{AUPR ($\uparrow$)}} & \multirow{2}{*}{\textbf{FPR95 ($\downarrow$)}} & \multirow{2}{*}{\textbf{ACC95 ($\uparrow$)}} & \multicolumn{4}{c}{\textbf{ACC@FPR$n$ ($\uparrow$)}}  \\
& & & & & & & & & 0 & 0.001 & 0.01 & 0.1\\
\midrule
 \parbox[c]{2mm}{\multirow{7}{*}{\rotatebox[origin=c]{90}{CIFAR10-LT}}}
& \xmark & & & $\ast \ast \ast$ & {84.91}  & 91.75  & {47.46}  & \textbf{92.91}  & 77.48  & 77.52  & 77.92  & 81.65  \\
& & \xmark & & $\ast \ast \ast$ & 96.10  & 98.02  & 16.87  & {83.53}  & 78.49  & 78.54  & 79.00  & 81.85   \\
&  &  & \xmark  &$\ast \ast \ast$  & {96.62}  & 98.22  & {14.92}  & 84.79  & {79.63}  & {79.69}  & {80.19}   & {83.47}   \\
& &  &  & $\ast$  & 97.08  & {98.70} & 13.86  & 84.09 & 80.38  & 80.41  & 80.72  & {83.10} \\
&  &  &  & $\ast \ast$  & 97.62  & {98.93}  & 11.43  & {83.87}  & 80.49  & 80.53  & 80.92  & {83.49} \\
 & \cellcolor[gray]{0.75} & \cellcolor[gray]{0.75} {EAT} & \cellcolor[gray]{0.75} & \cellcolor[gray]{0.75}& \cellcolor[gray]{0.75}\textbf{97.92}  & \cellcolor[gray]{0.75}\textbf{99.06}  & \cellcolor[gray]{0.75}\textbf{9.87}  & \cellcolor[gray]{0.75}84.39  & \cellcolor[gray]{0.75}\textbf{81.31}  & \cellcolor[gray]{0.75}\textbf{81.36}  &\cellcolor[gray]{0.75}\textbf{81.81}  &\cellcolor[gray]{0.75}\textbf{84.40}  \\
\midrule
 \parbox[c]{2mm}{\multirow{7}{*}{\rotatebox[origin=c]{90}{CIFAR100-LT}}}
& \xmark &  & &  $\ast \ast \ast$& {74.04}  & 84.99  & {63.58}  & \textbf{73.98} & \textbf{46.68}  & \textbf{46.70}  & \textbf{46.91}  & \textbf{49.50}  \\
&  & \xmark & &  $\ast \ast \ast$& {81.77}  & {89.67}  & 54.53  & {64.47} & 43.34  & 43.36  & 43.59  & 46.00  \\
&  &  & \xmark & $\ast \ast \ast$ & {79.52}  & 88.11  & {55.89}  & 64.30  & 43.93  & 43.94  & 44.10  & 46.24  \\
&  &  &  & $\ast$ & 80.70  & {88.27} & 52.86 & {64.30}  & 45.82 & 45.84  & 45.97  & {47.94} \\
&  &  &  & $\ast \ast$ & 82.13  & {89.01}  & 50.51  & 63.18 & {46.32}  & {46.32}  & {46.48}  & {48.57} \\
& \cellcolor[gray]{0.75} & \cellcolor[gray]{0.75} {EAT}  &  \cellcolor[gray]{0.75} &  \cellcolor[gray]{0.75} & \cellcolor[gray]{0.75}\textbf{83.11}  &\cellcolor[gray]{0.75}\textbf{89.71}  &\cellcolor[gray]{0.75}\textbf{47.78}  & \cellcolor[gray]{0.75}61.67  & \cellcolor[gray]{0.75}{46.23}  &  \cellcolor[gray]{0.75}{46.24}  & \cellcolor[gray]{0.75}{46.38}  & \cellcolor[gray]{0.75}{48.39}  \\
\bottomrule
\end{tabular}
}
\caption{The impact of key ingredients for EAT. Experiments are conducted on CIFAR10-LT and CIFAR100-LT ($\rho=100$). 
SVHN is used as $\douttest$. The number of $\ast$ denotes the ensemble size.} 
\label{tab:abla_dl_ft_cm}
\vspace{-1em}
\end{table*}
%

\paragraph{Improvements on head and tail classes}
In \Cref{tab:classwise}, we show the improvements of our method over OE on head and tail in-distribution classes. 
As we can see, our approach can substantially benefit both the head and tail classes. Compared with PASCL, it is highly biased towards the tail class and the improvement on head classes is marginal, our method achieves a good balance.

\paragraph{Why our method achieves low ACC@TPR$n$?}
We show the failure cases. \Cref{tab:cifar100-lt-0.01-ResNet18} and \Cref{tab:imagenet-lt-ResNet50} show several cases where the baseline is better than our approach concerning ACC@TPR$n$. We empirically find this is due to more in-distribution samples preserved by our approach than the baselines when a percentage of OOD samples have been successfully detected. As a result, the classification accuracy of the remaining samples may be lower even though our approach correctly classifies more in-distribution samples than baselines. We provide more statistics in the supplementary.

\subsection{Ablation Study}
\label{sec:ablation}
\paragraph{How do the key components of EAT affect the performance?}
In \Cref{tab:abla_dl_ft_cm}, we study the effects of the four critical components in our EAT approach: (1) virtual labels, (2) classifier fine-tuning, (3) CutMix, and (4) the mixture of experts (MoE), on CIFAR10-LT and CIFAR100-LT datasets.
Since the performance of most approaches fluctuates on SVHN, we choose SVHN as $\douttest$.

First, on both CIFAR10-LT and CIFAR100-LT datasets, employing the virtual label strategy for OOD data significantly improves OOD detection performance. It improves the AUROC, AUPR, and FPR95 by an average margin of 10\%. Note that, although without using virtual labels achieves higher results of ACC95, it is because more in-distribution samples are incorrectly deemed as OOD by the model. 
%

Second, fine-tuning improves the inlier classification while maintaining a competitive OOD detection performance. For instance, the ACC@FPR increases by about 3\% on average. Moreover, since we fine-tune the classifiers for only one iteration, it does not introduce much extra computational cost. 

Third, we study the role of CutMix. We find it beneficial to use context-rich images as backgrounds to improve the performance of OOD detection and inlier classification. For OOD detection, it improves AUROC and FPR95 by about 4\% and 8\% on CIFAR100-LT, respectively. For inlier classification, the ACC@FPR is improved by an average of 2\% on CIFAR100-LT. 
Notably, we observe different results on CIFAR100-LT and CIFAR10-LT concerning ACC@FPR$n$, which is related to the in-distribution accuracy. This may be related to the value of $k$ because a larger $k$ means that more abstention class heads need to be learned, resulting in a little effect on ID accuracy. 
Furthermore, \Cref{fig:ood-score} illustrates the distribution of OOD scores on the other three OOD datasets, and our model distinguishes OOD data from in-distribution data with a clear decision boundary.

\begin{figure}[t]
\resizebox{\linewidth}{!}{     
        \includegraphics[width=0.32\columnwidth]{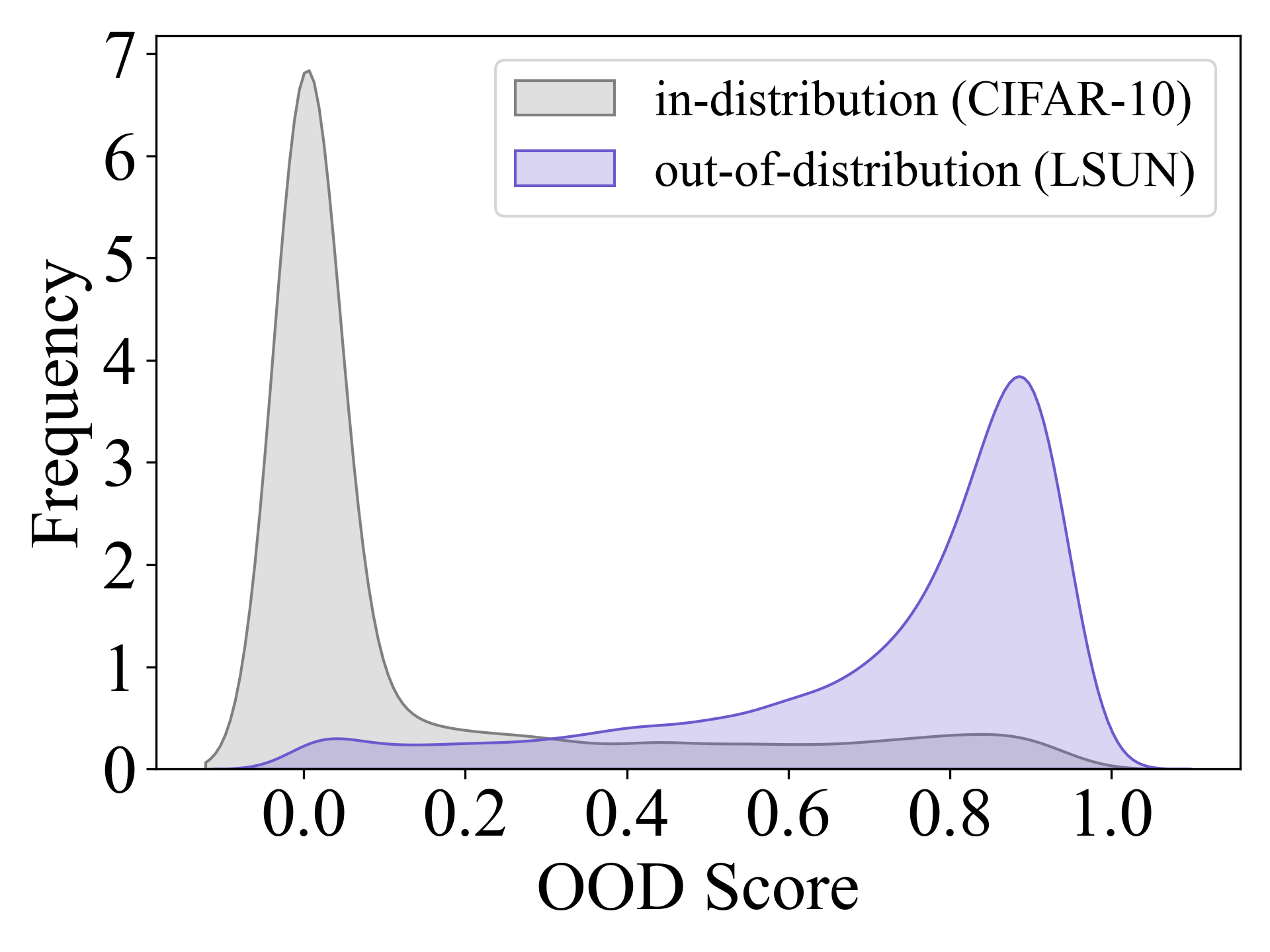}
        \includegraphics[width=0.32\columnwidth]{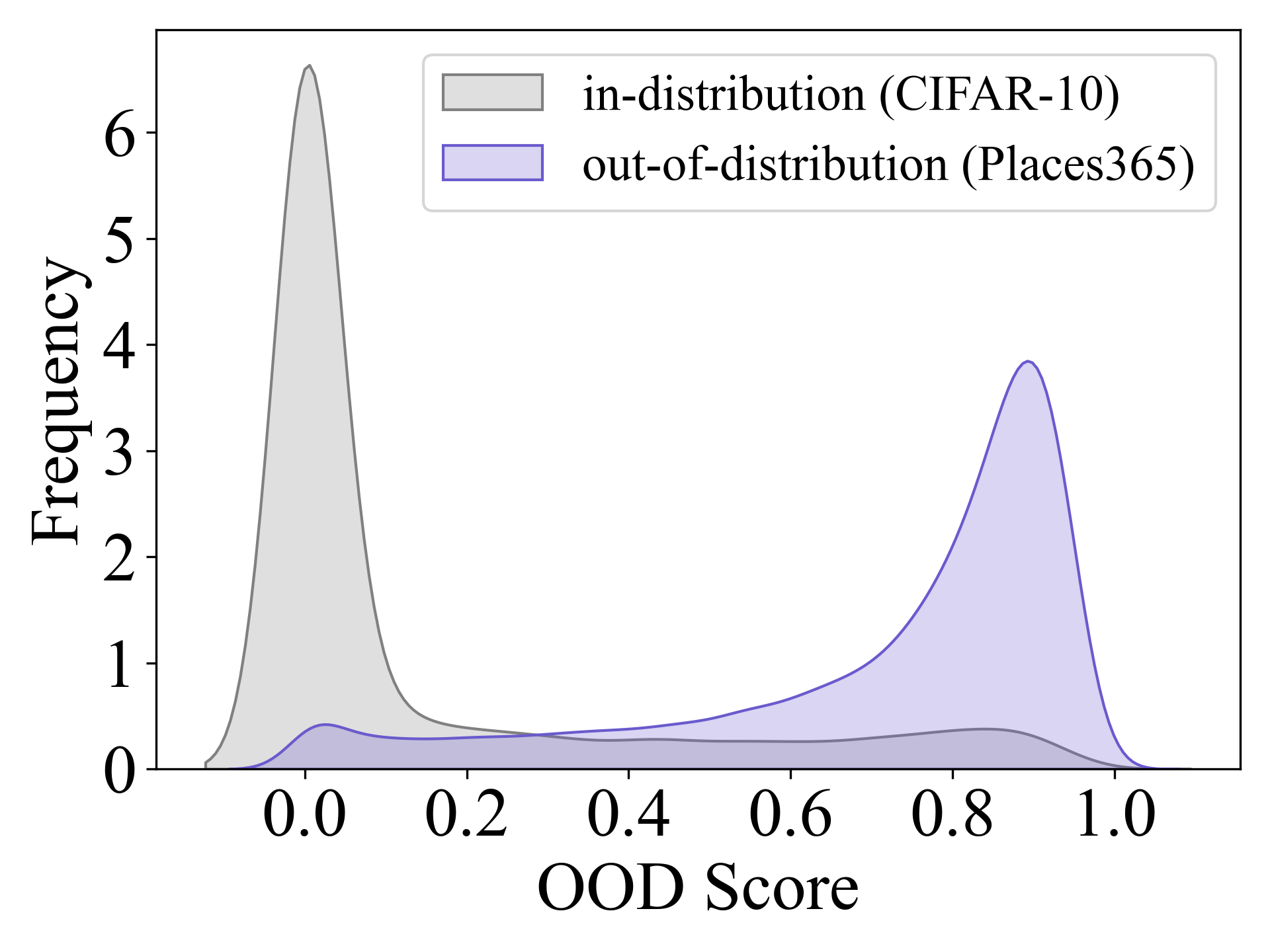}
         \includegraphics[width=0.32\columnwidth]{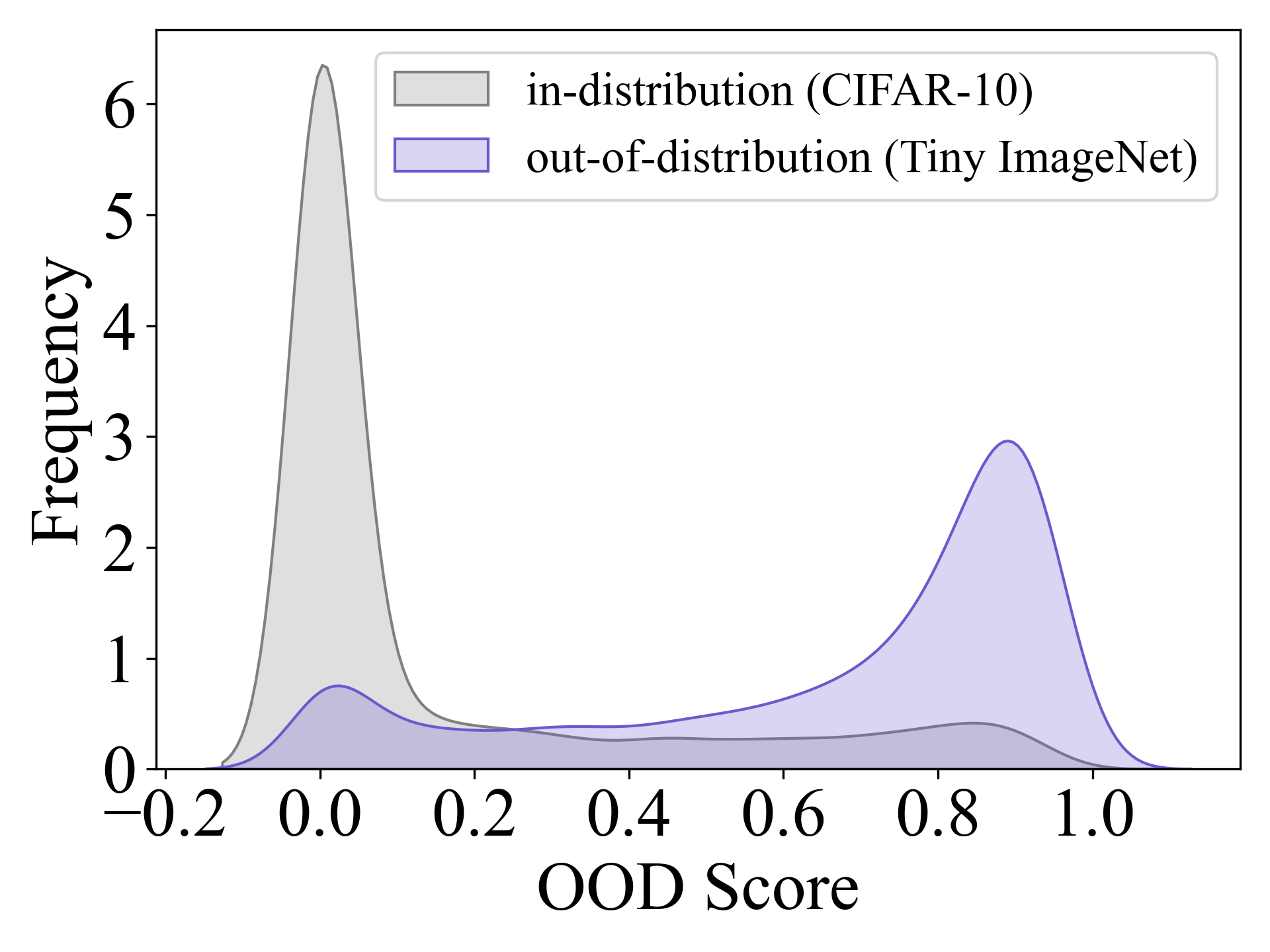}
}
\caption{Distribution of OOD scores from our model. The CIFAR10 is used as the in-distribution dataset, and the other six are OOD datasets. It shows that both in-distribution data and OOD data naturally form smooth distributions.}\label{fig:ood-score}
\vspace{-1em}
\end{figure}

\begin{table}[!h]
\centering

\resizebox{\linewidth}{!}{
\begin{tabular}{ ccccc }
\toprule
\textbf{Method} & \textbf{ACC ($\uparrow$)} & \textbf{AUROC ($\uparrow$)} & \textbf{AUPR ($\uparrow$)} & \textbf{FPR95 ($\downarrow$)} \\
\midrule
RIDE \cite{wang2020long} & \textbf{48.49}& {66.18}  &	{62.13} & 77.73 \\
\rowcolor[gray]{0.85}RIDE+Ours & 47.94 &  \textbf{72.41} & \textbf{69.37}  &  \textbf{71.99} \\
GLMC \cite{du2023global} & \textbf{54.51}& {65.01}  &	{62.01} &	79.65 \\
\rowcolor[gray]{0.85}GLMC+Ours & {52.30} & \textbf{73.07}  &  \textbf{67.14} &  \textbf{69.31}\\
\bottomrule
\end{tabular}
} 
\caption{Combining with other methods. The experiment is conducted on CIFAR-100 (in-distribution) dataset and six OOD datasets.}
\label{tab:plugin}
\end{table}
\vspace{-1em}
\paragraph{A Good Closed-Set Classifier is All You Need?} 
An interesting finding from previous work \cite{vaze2022openset} suggests that utilizing the maximum logit score rule with a highly accurate in-distribution classifier can outperform many well-designed OOD detectors. However, we sought to investigate whether this holds true in the context of long-tailed tasks. To validate this, we conducted experiments involving two sophisticated long-tail learning methods, namely RIDE \cite{wang2020long} and GLMC \cite{du2023global}. We find that their OOD detection performance lagged significantly behind that of our proposed method, providing evidence that a strong classifier alone is insufficient for effective long-tailed OOD detection. Furthermore, when we combined our method with RIDE and GLMC, as shown in Table \ref{tab:plugin}, we observe a substantial improvement in OOD detection performance with minimal sacrifice in in-distribution classification accuracy, underscoring the versatility of our approach.

\section{Conclusion}
In the real-world deployment of machine learning models, test inputs with previously unseen classes are often encountered. For safety, it may be essential to identify such inputs. Moreover, the class-balanced training in-distribution data assumption rarely holds in the wild.
This paper proposes a novel framework (EAT) to tackle the long-tailed OOD detection problem. Towards this end, EAT presents several general techniques that can easily be applied to mainstream OOD detectors and long-tail learning methods. First, the abstention OOD classes can be used as an alternative to the outlier exposure method. Second, tail-class augmentation can be employed as a universal add-on for existing methods. Third, the classifier ensembling technique can further boost the performance without introducing much additional computational cost. Finally, we evaluate the proposed method on many commonly used datasets, showing that it consistently outperforms the existing state-of-the-art. 


\bibliography{arxiv}

\appendix

\section{ Proof of Proposition 1 }
\begin{prop}
For the cross-entropy loss, Eq. (3) induces gradient noise $\boldsymbol{g} = - \frac{\nabla_{\boldsymbol{\theta}}\boldsymbol{\widetilde{z}}_j}{\boldsymbol{\widetilde{z}}_j}$ on $\nabla_{\boldsymbol{\theta}}\ell(\boldsymbol{z},y)$, s.t., $\boldsymbol{g} \in \mathbb{R}^{p}, j = \arg \max_{j\in [C+1,C+k]} \boldsymbol{\widetilde{z}}$. While each OOD sample in OE \cite{hendrycks2018deep} induces gradient noise $\boldsymbol{g}^\prime = -\frac{1}{C}\sum_{j=1}^{C}\frac{\nabla_{\boldsymbol{\theta}}\boldsymbol{\widetilde{z}}_j}{\boldsymbol{\widetilde{z}}_j}$ on $\nabla_{\boldsymbol{\theta}}\ell(\boldsymbol{z},y)$, where $\frac{\cdot}{\boldsymbol{\widetilde{z}}_j}$ denotes the element-wise division.
\end{prop}

\begin{proof}
    For Eq.~(1), the total gradient is as follows:
\begin{align*}
    \widetilde{\nabla}_{\boldsymbol{\theta}}\ell_{\mathrm{total}} &= \nabla_{\boldsymbol{\theta}}\ell\left(\boldsymbol{z},y\right) + \nabla_{\boldsymbol{\theta}}\ell\left(\widetilde{\boldsymbol{z}},\widetilde{y}\right) \\&=\nabla_{\boldsymbol{\theta}}\ell\left(\boldsymbol{z},y\right) +  \nabla_{\widetilde{\boldsymbol{z}}}\ell\left(\widetilde{\boldsymbol{z}},\widetilde{y}\right)\cdot\nabla_{\boldsymbol{\theta}}\widetilde{\boldsymbol{z}}
\end{align*}
We omit the trade-off parameter $\lambda$ in Eq.~(1) for simplicity.
Note that cross-entropy loss is $\ell\left(\boldsymbol{z},y\right) = -\boldsymbol{e}^{y} \log \boldsymbol{z}$, then the noise imposed on the in-distribution classification loss $\nabla_{\boldsymbol{\theta}}\ell\left(\boldsymbol{z},y\right)$ is:
\begin{align*}
    \boldsymbol{g} &= \nabla_{\widetilde{\boldsymbol{z}}}\ell\left(\widetilde{\boldsymbol{z}},\widetilde{y}\right)\cdot\nabla_{\boldsymbol{\theta}}\widetilde{\boldsymbol{z}} = -\left(\frac{e^{\widetilde{y}}}{\widetilde{\boldsymbol{z}}}\right)^{T}\cdot\nabla_{\boldsymbol{\theta}}\widetilde{\boldsymbol{z}}\\& = - \sum_{j=1}^{C+k} \left(e^{\widetilde{y}}_j \cdot \frac{\nabla_{\boldsymbol{\theta_i}}\boldsymbol{\widetilde{z}}_j}{\boldsymbol{\widetilde{z}}_j}\right), \quad \text{s.t. } \widetilde{y} = \arg \max_{j\in [C+1,C+k]} \boldsymbol{\widetilde{z}}.
\end{align*}
Since $\boldsymbol{e}^{\widetilde{y}} = (0, \cdots, 1, \cdots, 0)$ is the one-hot vector and only the $y$-th entry is 1, the expression of the noise $\boldsymbol{z}$ can be simplified as:
\begin{equation*}
    \boldsymbol{g} = -\frac{\nabla_{\boldsymbol{\theta}}\boldsymbol{\widetilde{z}}_j}{\boldsymbol{\widetilde{z}}_j}, \quad \text{s.t. } j = \arg \max_{j\in [C+1,C+k]} \boldsymbol{\widetilde{z}}.
\end{equation*}

Let $g_i$ be the $i$-th entry of $\boldsymbol{g}$, we have
\begin{equation*}
    g_i = -\frac{\nabla_{\boldsymbol{\theta}_i}\boldsymbol{\widetilde{z}}_j}{\boldsymbol{\widetilde{z}}_j}, \quad \text{s.t. } j = \arg \max_{j\in [C+1,C+k]} \boldsymbol{\widetilde{z}}.
\end{equation*}

While the regularization item of the OE method is $ \ell_{\mathrm{OE}} = - \frac{1}{C}\cdot\sum_{i=1}^{C}\log \boldsymbol{\widetilde{z}}_i$. 

Then the gradient of $\ell_{\mathrm{OE}}$ w.r.t $\theta$ is as follows:

\begin{align*}
\nabla_{\theta} \ell_{\mathrm{OE}} &= \nabla_{\boldsymbol{\widetilde{z}}}\ell_{\mathrm{OE}} \cdot \nabla_{\theta}\boldsymbol{\widetilde{z}} = \nabla_{\boldsymbol{\widetilde{z}}}\left(- \frac{1}{C}\cdot\sum_{j=1}^{C}\log \boldsymbol{\widetilde{z}}_j\right) \cdot \nabla_{\theta}\boldsymbol{\widetilde{z}} \\&= - \frac{1}{C}\cdot\sum_{j=1}^{C}\frac{\nabla_{\theta}\boldsymbol{\widetilde{z}}_j}{\boldsymbol{\widetilde{z}}_j}
\end{align*}
\end{proof}


\section{Additional Experimental Results}
\label{sec:appx-results}
\paragraph{On imbalance ratio $\rho$}

We use imbalance ratio $\rho=100$ on both CIFAR10-LT and CIFAR100-LT in routine long-tail OOD experiments. In this section, we show that our method can work well under different imbalance ratios. 
Specifically, we conduct experiments on CIFAR10-LT with $\rho=50$. The results are shown in \Cref{tab:cifar10-lt-0.05-ResNet18}. Our approach also outperforms the OE by a considerable margin when $\rho=50$ and outperforms previous state-of-the-art PASCL in 3 out of 4 cases. 

\begin{table}[!h]
\centering
\caption{Results on CIFAR10-LT ($\rho=50$) using ResNet18.}
\resizebox{1\linewidth}{!}{
\begin{tabular}{ c|c|ccc|c }
\toprule
$\douttest$ & \textbf{Method} & \textbf{AUROC ($\uparrow$)} & \textbf{AUPR ($\uparrow$)} & \textbf{FPR95 ($\downarrow$)} & \textbf{ACC ($\uparrow$)} \\
\midrule
\midrule
\multirow{3}{*}{Average} 
& OE & 93.13 & 91.06 & 24.73 & 83.34 \\
& PASCL & {93.94} & {92.79} & \textbf{22.80} & {85.44} \\
& \ours & \textbf{94.06} & \textbf{93.50} & {24.03} & \textbf{85.57} \\
\midrule
\bottomrule
\end{tabular}
} 
\label{tab:cifar10-lt-0.05-ResNet18}
\end{table}

\paragraph{On model structures}
In routine experiments, we use the standard ResNet18 as the backbone model. In this section, we show that our method can work well under different model structures, but conduct experiments using the standard ResNet34 \cite{he2016deep}. The results are shown in \Cref{tab:cifar10-lt-0.01-ResNet34}. Our method consistently outperforms OE and PASCL. In particular, it improves the in-distribution accuracy by 6\% in comparison with PASCL.
\begin{table}[!h]
\centering
\caption{Results on CIFAR10-LT ($\rho=100$) using ResNet34.}
\resizebox{1\linewidth}{!}{
\begin{tabular}{ c|c|ccc|c }
\toprule
$\douttest$ & \textbf{Method} & \textbf{AUROC ($\uparrow$)} & \textbf{AUPR ($\uparrow$)} & \textbf{FPR95 ($\downarrow$)} & \textbf{ACC ($\uparrow$)} \\
\midrule
\midrule
\multirow{3}{*}{Average} 
& OE & 89.86 & 87.28 & 33.66 & 73.39 \\
& PASCL & {91.11} & {89.28} & {33.21} & {75.34} \\
& \ours & \textbf{93.38} & \textbf{92.85} & \textbf{26.56} & \textbf{82.20} \\
\midrule
\bottomrule
\end{tabular}
} 
\label{tab:cifar10-lt-0.01-ResNet34}
\end{table}

\begin{figure*}[ht]
    \centering
         \includegraphics[width=0.5\columnwidth]{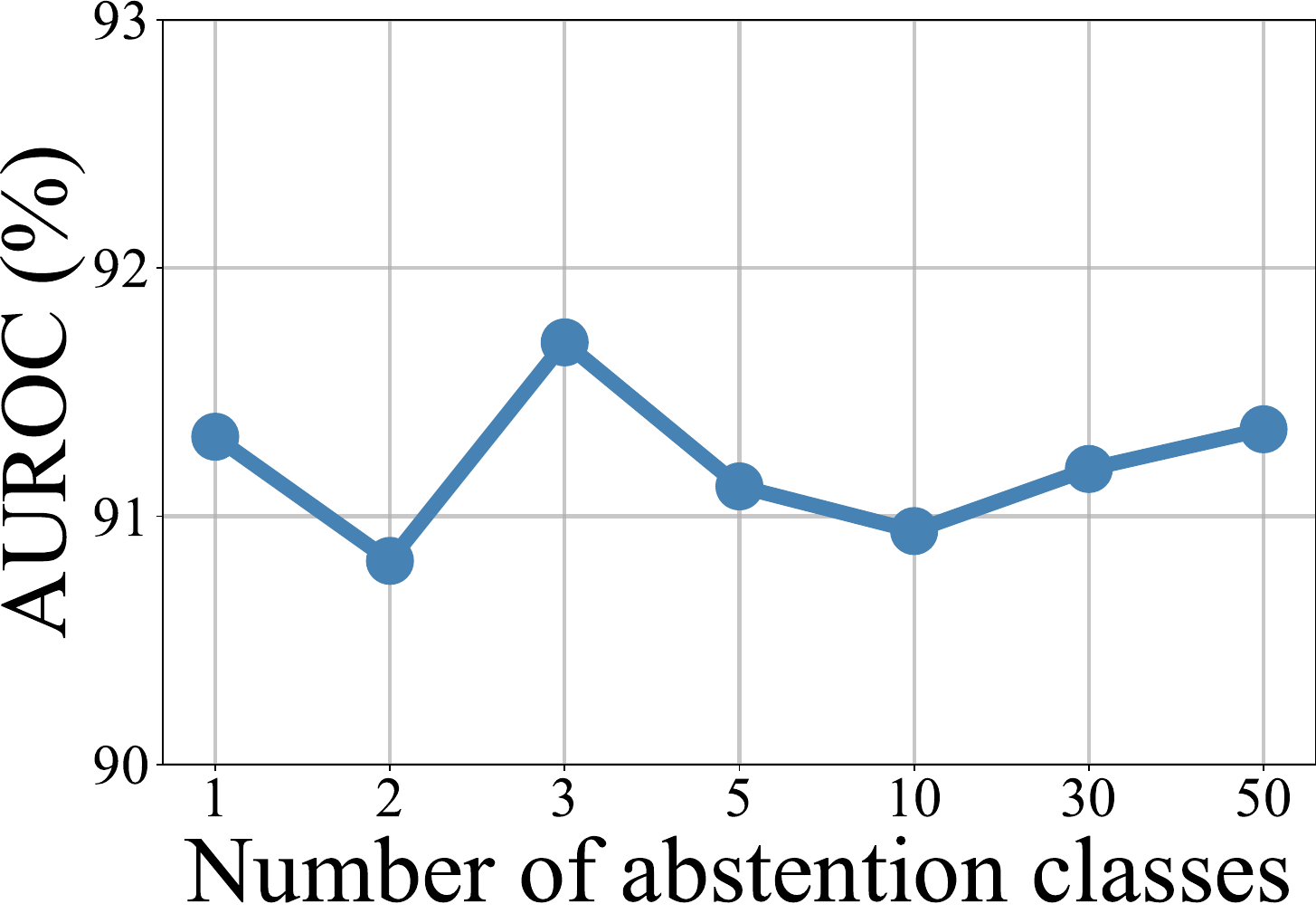} 
         \includegraphics[width=0.5\columnwidth]{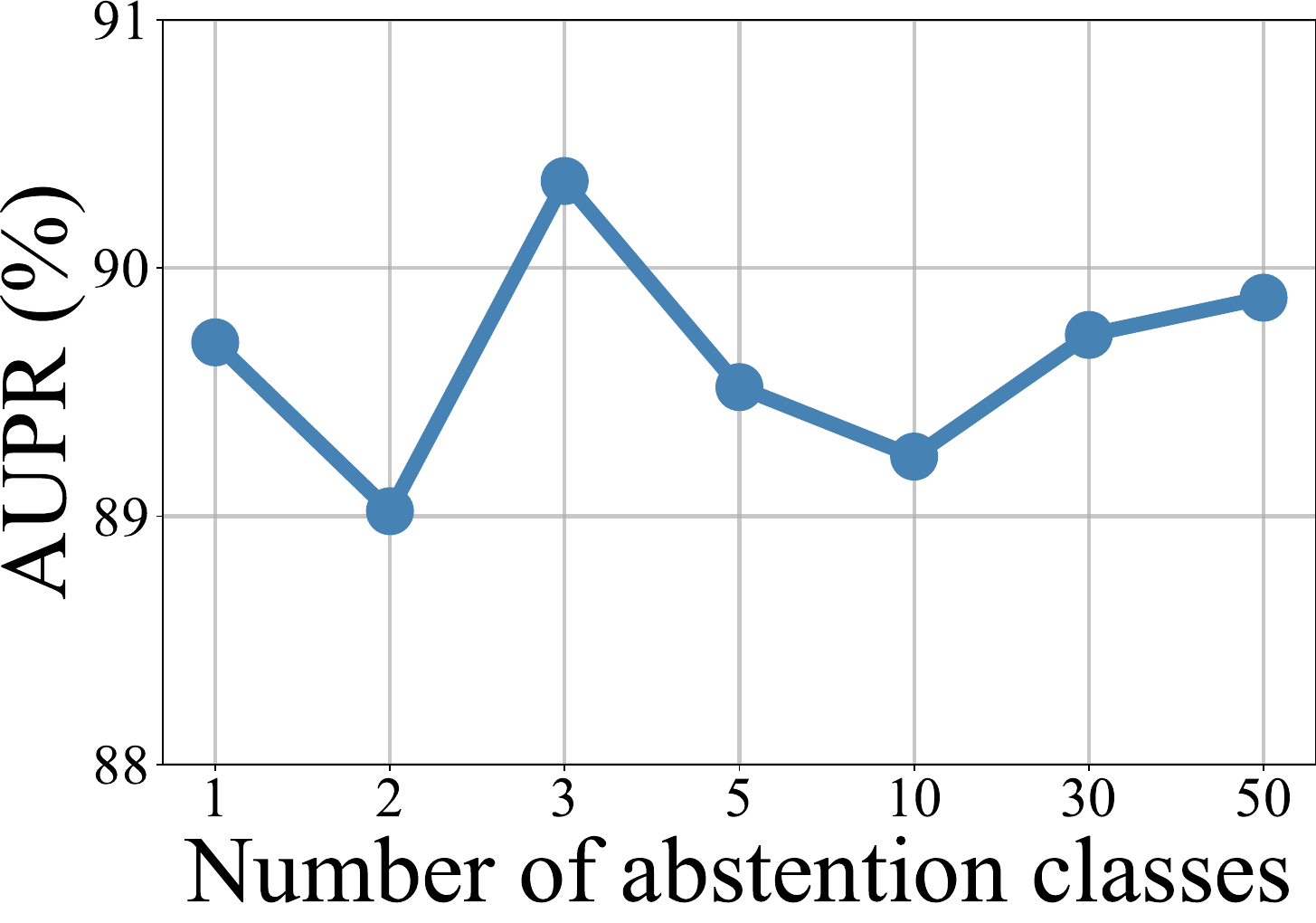}
         \vspace{1em}
         \includegraphics[width=0.5\columnwidth]{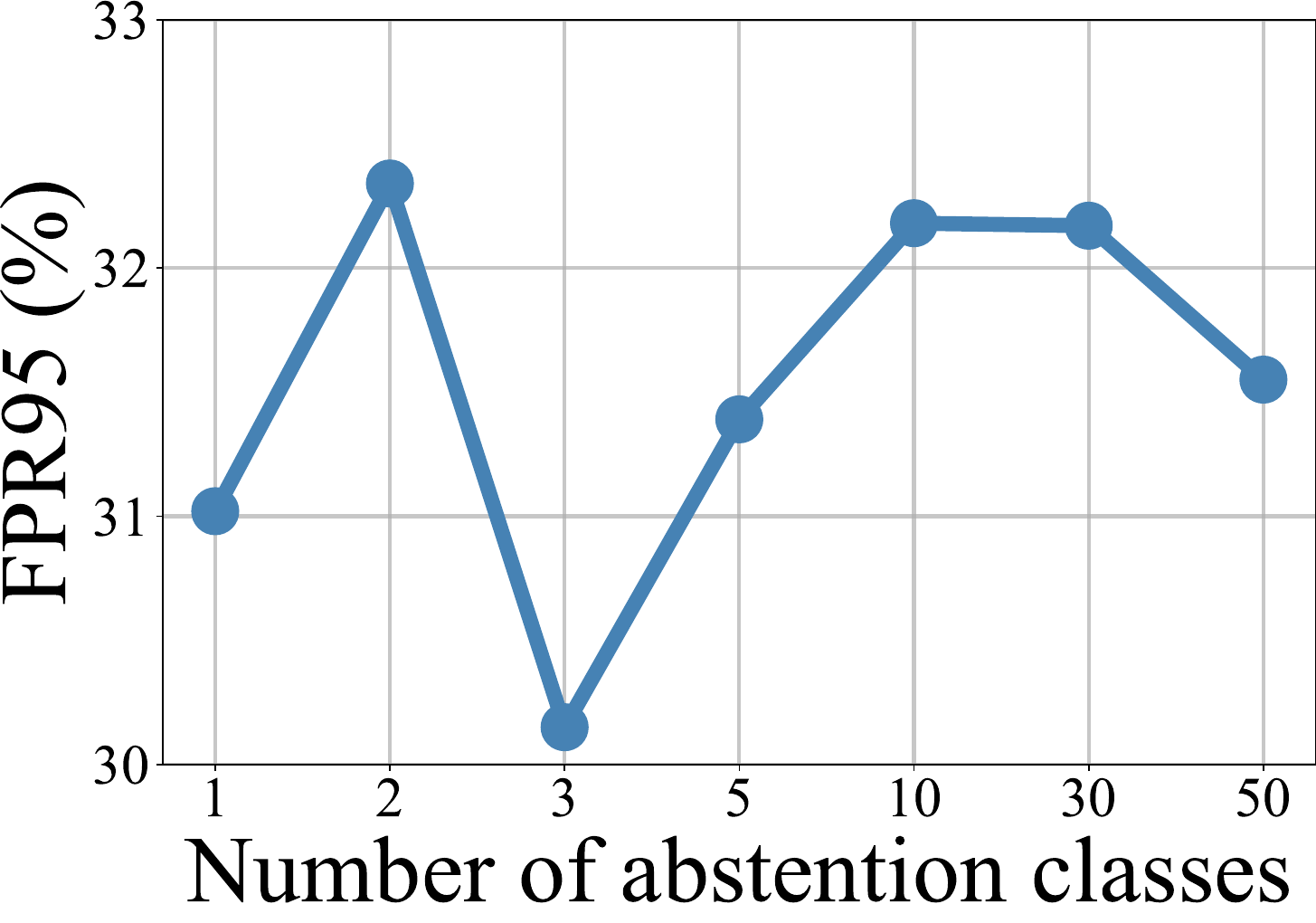} 
         \includegraphics[width=0.5\columnwidth]{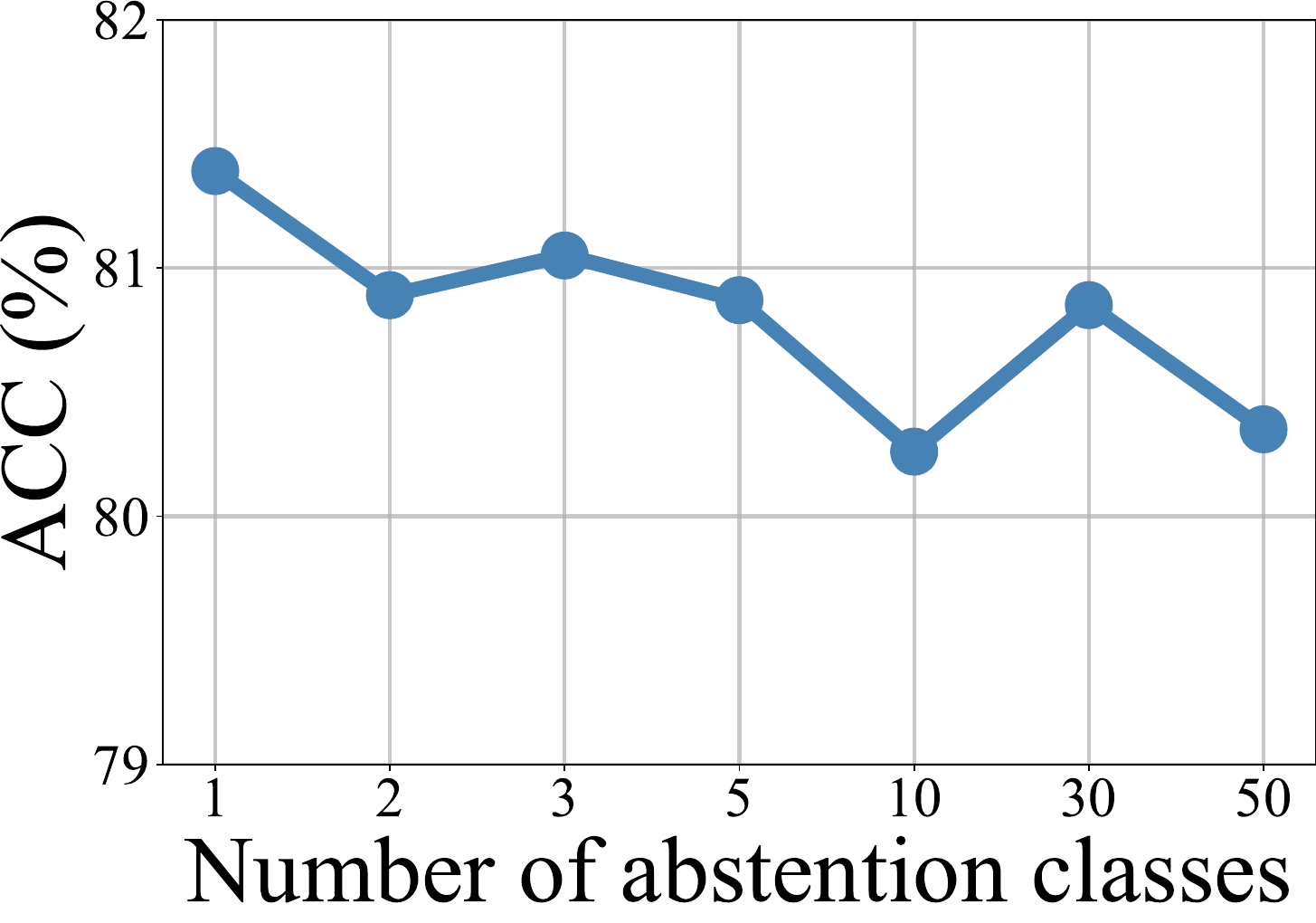}
\caption{Effect of the number of abstention classes on the model performance.}\label{ablation-ood-class}
\end{figure*}
\paragraph{How does the number of abstention classes affect the performance?}
\Cref{ablation-ood-class} shows how the performance changes with different numbers of abstention classes on CIFAR10-LT. Overall the performance is not sensitive in the range chosen. We empirically find that setting $k=3$ achieves relatively better results with respect to three out of four representative OOD detection performance measures, i.e., AUROC, AUPR, and FPR95.

\paragraph{More visualization}
\Cref{fig:appedix-ood-score} illustrates the distribution of OOD scores on the other three OOD datasets considered in this paper. Overall our model distinguishes OOD data from in-distribution data with a clear decision boundary.

\paragraph{Convergence and instability of our method.}
We present the training loss curve and test accuracy in Figure \ref{fig:rebuttal-1}, providing evidence of the stability of our method throughout training. This visualization serves to assure that overfitting is not observed in our approach.

\vspace{-0.5em}
\begin{figure}[h]
\resizebox{\linewidth}{!}{     
        \includegraphics[width=1\columnwidth]{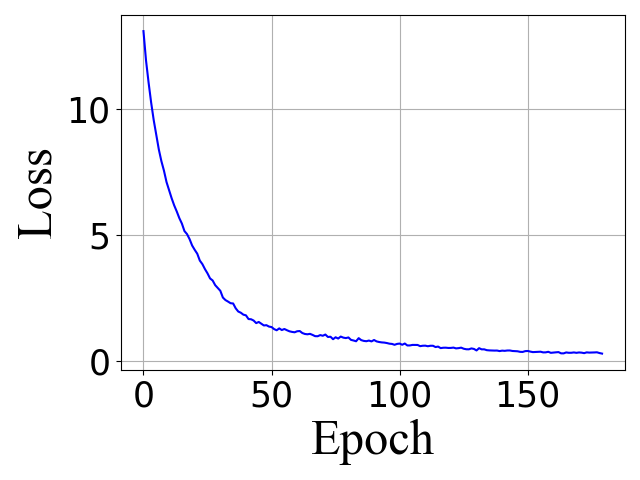}
        \includegraphics[width=1\columnwidth]{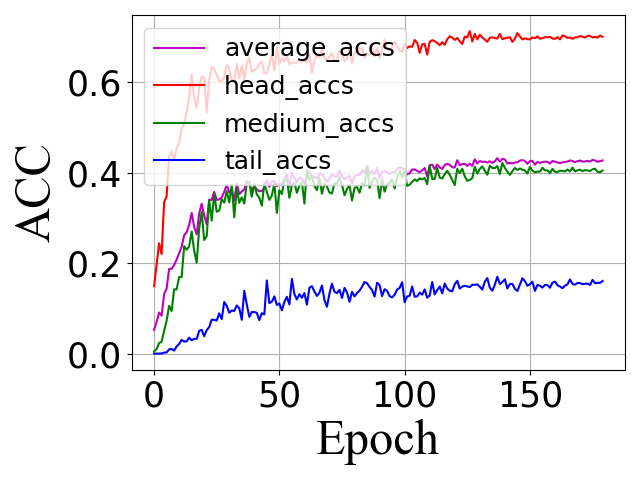}
}
\vspace{-1.5em}
\caption{Training loss curve and test accuracy.}\label{fig:rebuttal-1}
\end{figure}

\section{Contrast with SOFL}
\label{sec:appx-related}
Similar to our approach, SOFL \cite{mohseni2020self} also uses multiple abstention classes for OOD detection and fine-tunes the model in the second stage. Specifically, SOFL trains on in-distribution data using the cross-entropy loss in the first stage, then uses both in- and out-of-distribution data in the second stage to fine-tune the whole model. 

However, as demonstrated in Algorithm 1, we train the model in the first stage using in- and out-of-distribution data. In the second stage, we only fine-tune the final classification layer of the model using in-distribution data and the loss $\ell_{\text{LA}}$ specially designed for long-tailed learning.

Our approach surpasses SOFL in terms of in-distribution classification and OOD detection by a large margin. It is not just about other aspects of our model, but more intuitively, we have superior designs in the two training stages. We use more data in the first stage to make the representation learning of the model more discriminative. The logit adjustment loss adopted in the second stage further boosts the accuracy for tail classes. Moreover, only updating the final classification layer also preserves the model's judgment ability for OOD.

\begin{figure*}[t]
\resizebox{\linewidth}{!}{     
        \includegraphics[width=0.32\columnwidth]{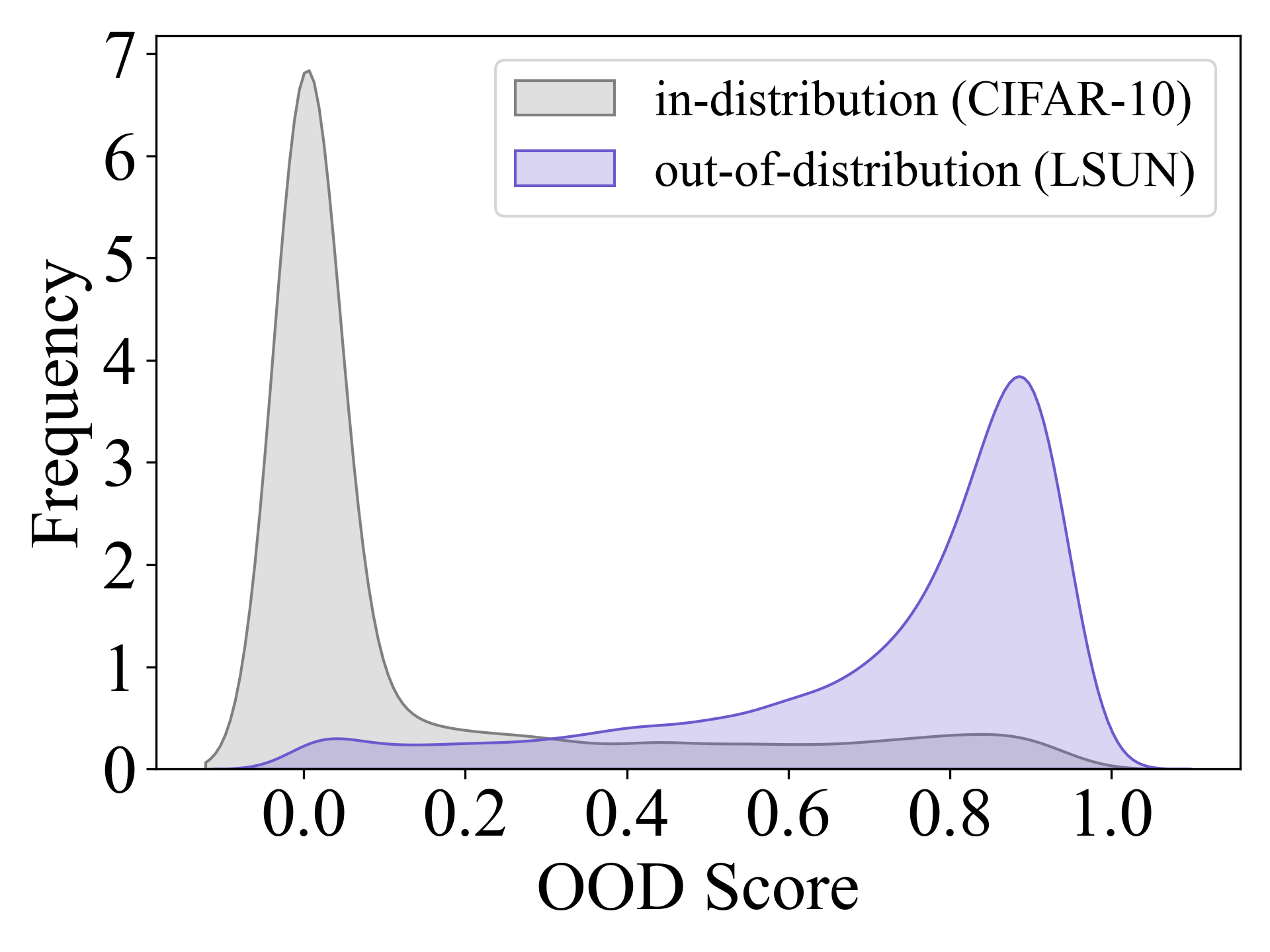}
        \includegraphics[width=0.32\columnwidth]{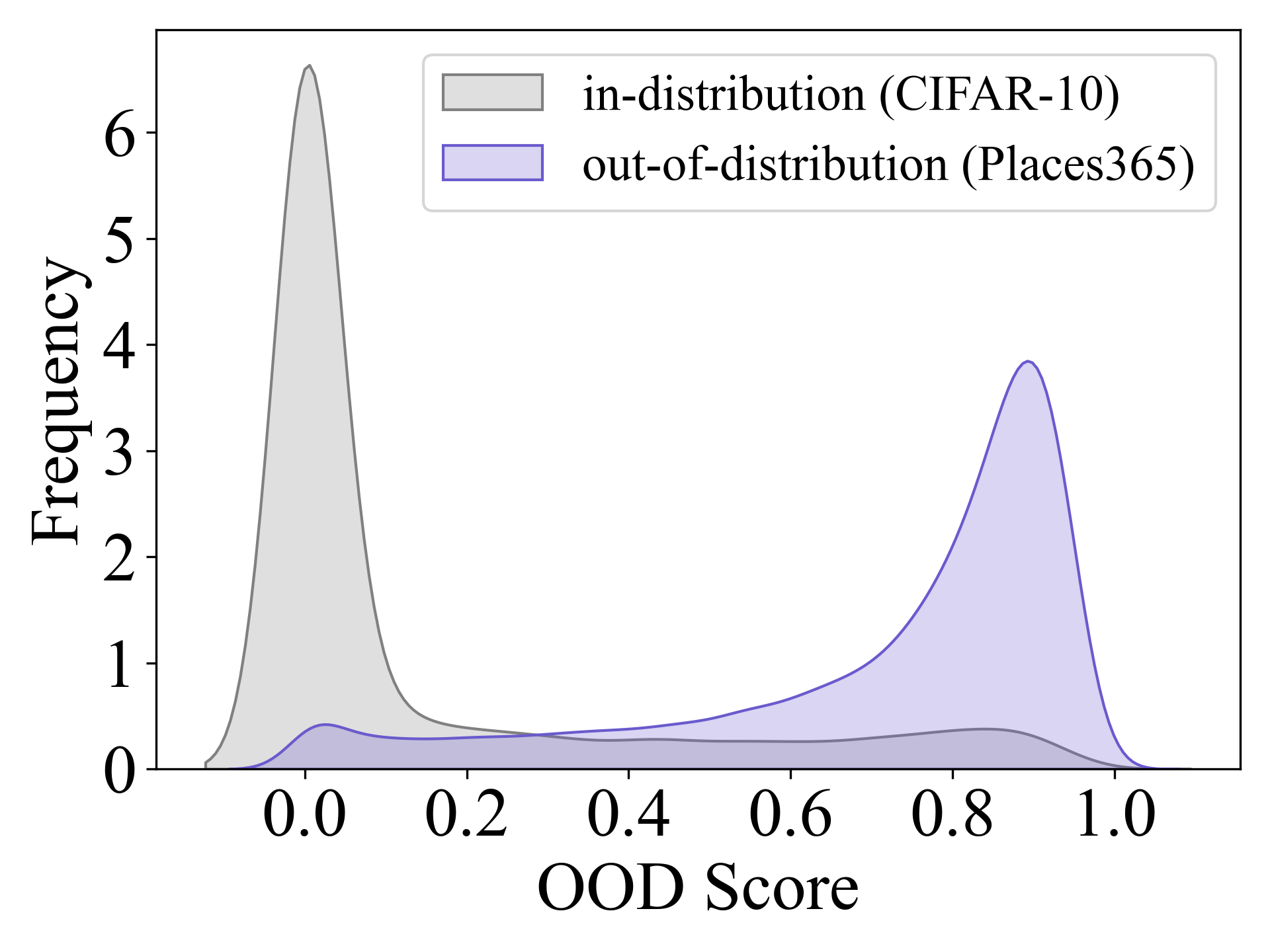}
         \includegraphics[width=0.32\columnwidth]{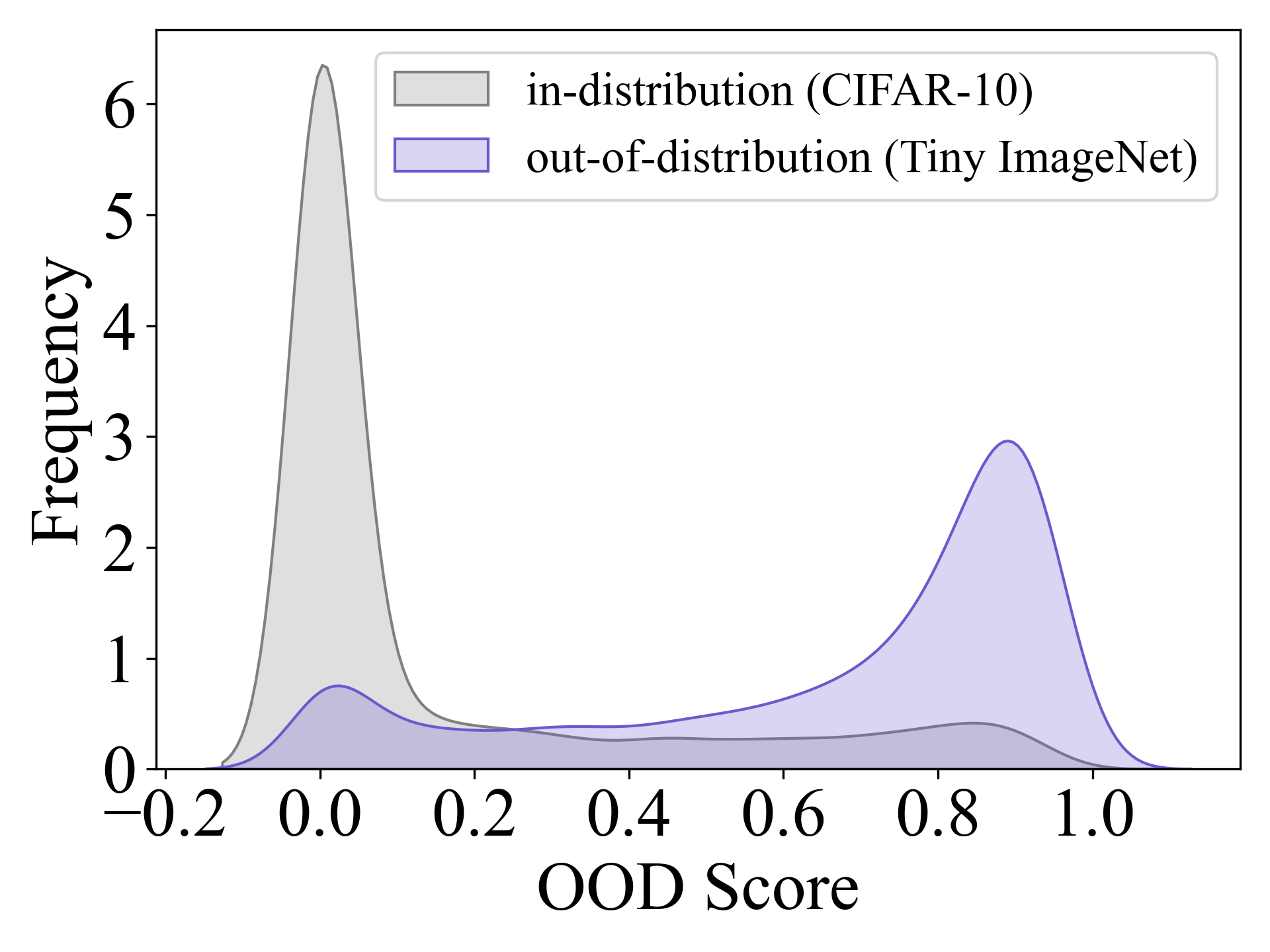}
}
\caption{Distribution of OOD scores from our model. The CIFAR10 is used as the in-distribution dataset, and the other six are OOD datasets. It shows that both in-distribution data and OOD data naturally form smooth distributions.}\label{fig:appedix-ood-score}
\end{figure*}

\begin{table}[!h]
\centering
\caption{We compare the ACC95 of PASCL and ours. Since the number of $\dintest$ is $10,000$, we round the product of these numbers to represent the number of correctly classified in-distribution samples in the remaining in-distribution samples when $95$\% percent of the OOD samples are selected. ACC95 and $\text{1-FPR95}$ are shown in percentages.}
\resizebox{1\linewidth}{!}{
\begin{tabular}{ c|c|ccc }
\toprule
$\douttest$ & \textbf{Method} & \textbf{ACC95 ($\uparrow$)} & \textbf{1 - FPR95 ($\uparrow$)} &  $N_{\text{correct}}$ ($\uparrow$) \\
\midrule
\midrule
\multirow{3}{*}{\makecell{Texture}} & OE & 71.43 & 31.72 & 2266  \\
& PASCL & 73.11 & 32.57 & 2381  \\
& Ours & 73.76  & 32.47  & 2395   \\
\midrule
\multirow{3}{*}{SVHN} & OE & 64.27 & 41.96 & 2697  \\
& PASCL & 64.50 & 46.55 & 3002  \\
& Ours & 61.67  & 52.22  & 3220   \\
\midrule
\multirow{3}{*}{CIFAR10} & OE & 82.67 & 19.36 & 1601  \\
& PASCL & 82.30 & 20.45 & 1683  \\
& Ours & 82.61  & 22.03  & 1820   \\
\midrule
\multirow{3}{*}{Tiny ImageNet} & OE & 76.22 & 23.34 & 1779  \\
& PASCL & 77.56 & 23.89 & 1853  \\
& Ours & 77.07  & 25.11  & 1935   \\
\midrule
\multirow{3}{*}{LSUN} & OE & 65.64 & 36.02 & 2364  \\
& PASCL & 68.05 & 36.69 & 2497  \\
& Ours & 62.07  & 44.98  & 2791   \\
\midrule
\multirow{3}{*}{Places365} & OE & 67.04 & 34.28 & 2298  \\
& PASCL & 69.04 & 35.19 & 2430  \\
& Ours & 66.15  & 39.15  & 2590   \\
\midrule
\multirow{3}{*}{Average} & OE & 71.21 & 31.11 & 2215  \\
& PASCL & 72.43 & 32.56 & 2358  \\
& Ours & 70.55  & 36.00  & 2540   \\
\midrule
\bottomrule
\end{tabular}
}
\label{tab:acc95}
\end{table}

\section{Balance of FPR95 and ACC95}\label{appedix:acc95}
Mohseni et al. \cite{mohseni2020self} show that it is very challenging to achieve high FPR95 and ACC95 concurrently in the OOD detection task. The reason is simple: when FPR95 is low, more in-distribution samples are correctly detected, including those problematic and corner-case in-distribution samples. As a result, the remaining in-distribution samples might incur more misclassification, leading to low ACC95.

In this paper, we raise a question about ACC95: as an indicator of OOD detection, whether a larger ACC95 is preferable when $n$ (in percentage) OOD samples have been successfully detected. Because when a fixed rate of OOD samples is filtered out, the number of the remaining in-distribution samples is not stable among different algorithms or experiment settings. When this number does not fluctuate much among various algorithms, a higher ACC95 indicates better performance. However, there is an alternative interpretation when the number of remaining inliers fluctuates significantly. We compare the results with the previous long-tail OOD detection method PASCL \cite{wang2022pascl}, and we find that although our approach performs worse than PASCL concerning ACC95 on most of the OOD datasets, the number of remaining in-distribution samples is far more than PASCL. We believe that with the current experiment setups, it makes more sense to compare the number of remaining inlier samples correctly classified after filtering out $n$ (in percentage) OOD samples. The detailed results are shown in \cref{tab:acc95}.

Therefore, we may combine the evaluation of these two indicators. Assuming the number of $\dintest$ is $N$, the number of correctly classified inlier samples remaining when $95$\% percent of the OOD samples are filtered out is $N_{\text{correct}}$. Then we have the following formula:
\begin{equation}
N_{\text{correct}} = N \times (1-FPR95) \times ACC95 
\label{eq:Calculate N}
\end{equation}
In this way, we may evaluate ACC95 and FPR95 in the meantime.
It is important to note that $N_{\text{correct}}$ has a practical physical significance. When $95\%$ of the OOD samples are filtered out, the remaining correctly classified in-distribution samples.
Considering the influence of $N$ on this detection value, $N$ can be omitted when comparing different experiments, and the product of ACC95 and (1-FPR95) can be directly calculated as a measurement.
In fact, our method optimizes both ACC95 and FPR95 on CIFAR10-LT. And even though ACC95 performs slightly worse on CIFAR100-LT, the number of the remaining in-distribution samples $N_{\text{correct}}$ still outperforms OE and PASCL on each $\douttest$ and overall. In addition, the calculation process of $N_{\text{correct}}$ involves both the inlier accuracy of the model and the sensitivity to OOD detection, which may make $N_{\text{correct}}$ become a new paradigm with both in-distribution classification and OOD detection.

\section{Limitations}


While our method examines tail-class augmentation via CutMix, which uses head-class and OOD images as background and tail-class images as the foreground, it results in more reasonable and effective representation learning for tail samples, greatly improving tail part in-distribution accuracy and OOD detection performance. However, since CutMix generates more data, it necessitates additional GPU RAM, albeit with a continuous rise in size. We attempted to minimize extra time overhead and GPU memory consumption by appropriate code architecture and algorithm structure, but we must agree that some extra overhead is unavoidable.

Additionally, this paper focuses on improving the performance by directly using CutMix, that is, we do not adjust the sampling area distribution or fine-tune any other component parts to make this technique more suitable for long-tailed OOD tasks. And we do not rule out the possibility of other more advanced methods achieving this function and yielding better results in the future.


In many real-world applications such as autonomous driving, medical diagnosis, and healthcare, beyond being naturally imbalanced. Moreover, the model usually encounters unknown (out-of-distribution) test data points once deployed. In this paper, we focus on standard classification and out-of-distribution detection as our measure and largely ignore other ethical issues in imbalanced data, especially in minor classes. For example, the data may impose additional constraints on the learning process and final models, e.g., being fair or private. As such, the risk of producing unfair or biased outputs reminds us to carry rigorous validations in critical, high-stakes applications.

\end{document}